\documentclass[twoside]{article}
\usepackage[accepted]{aistats2019}

\usepackage{natbib}
\setcitestyle{authoryear,open={(},close={)}}
\usepackage{amsmath}
\usepackage{amsthm}
\usepackage{amssymb}
\usepackage{algorithm}
\usepackage{color}
\usepackage[english]{babel}
\usepackage{graphicx}
\usepackage{wrapfig,epsfig}
\usepackage{epstopdf}
\usepackage{url}
\usepackage{graphicx}
\usepackage{color}
\usepackage{epstopdf}
\usepackage{algpseudocode}
\usepackage{scrextend}
\usepackage[T1]{fontenc}
\usepackage{bbm}
\usepackage{comment}
\usepackage[subrefformat=parens]{subfig}
\usepackage{setspace} 

\allowdisplaybreaks

\usepackage{tikz}
\usepackage{hyperref}  
\hypersetup{colorlinks=true,citecolor=blue,linkcolor=blue} 
\usetikzlibrary{arrows}
\graphicspath{{./figs/}}

\date{}

\newtheorem{theorem}{Theorem}[section]
\newtheorem{lemma}[theorem]{Lemma}
\newtheorem{definition}[theorem]{Definition}

\newtheorem{fact}[theorem]{Fact}
\newtheorem{remark}[theorem]{Remark}
\newtheorem{claim}[theorem]{Claim}

\newtheorem{property}[theorem]{Property}

\newcommand{\wh}{\widehat}
\newcommand{\wt}{\widetilde}

\newcommand{\R}{\mathbb{R}}

\renewcommand{\varepsilon}{\epsilon}
\renewcommand{\tilde}{\wt}
\renewcommand{\hat}{\wh}

\newcommand{\cU}{\mathcal{U}}

\DeclareMathOperator*{\E}{{\bf {E}}}

\DeclareMathOperator{\poly}{poly}

\DeclareMathOperator{\nnz}{nnz}

\DeclareMathOperator{\rank}{rank}

\definecolor{darkgreen}{RGB}{0,100,0}

\makeatletter
\newcommand*{\RN}[1]{\expandafter\@slowromancap\romannumeral #1@}
\makeatother

\usepackage{lineno}

\usepackage{etoolbox}

\makeatletter

\newcommand{\define}[4][ignore]{%
  \ifstrequal{#1}{ignore}{}{
  \@namedef{thmtitle@#2}{#1}}%
  \@namedef{thm@#2}{#4}%
  \@namedef{thmtypen@#2}{lemma}%
  \newtheorem{thmtype@#2}[theorem]{#3}%
  \newtheorem*{thmtypealt@#2}{#3~\ref{#2}}%
}

\newcommand{\state}[1]{%
  \@namedef{curthm}{#1}
  \@ifundefined{thmtitle@#1}{
  \begin{thmtype@#1}
    }{
  \begin{thmtype@#1}[\@nameuse{thmtitle@#1}]
  }
    \label{#1}
    \@nameuse{thm@#1}
  \end{thmtype@#1}
  \@ifundefined{thmdone@#1}{
  \@namedef{thmdone@#1}{stated}%
  }{}
}

\newcommand{\restate}[1]{%
  \@namedef{curthm}{#1}
  \@ifundefined{thmtitle@#1}{
    \begin{thmtypealt@#1}
    }{
  \begin{thmtypealt@#1}[\@nameuse{thmtitle@#1}]
  }
    \@nameuse{thm@#1}
  \end{thmtypealt@#1}
  \@ifundefined{thmdone@#1}{
  \@namedef{thmdone@#1}{stated}%
  }{}
}

\newcommand{\thmlabel}[1]{
  \@ifundefined{thmdone@\@nameuse{curthm}}{\label{#1}
    }{\tag*{\eqref{#1}}}
}
\makeatother

\begin{document}

\runningauthor{Lin, Song, Yang}

\twocolumn[
\aistatstitle{Towards a Theoretical Understanding of Hashing-Based Neural Nets}

\aistatsauthor{
  Yibo Lin\\
  \texttt{yibolin@utexas.edu}
  \And
  Zhao Song \\
  \texttt{zhaos@g.harvard.edu}
  \And
  Lin F. Yang \\
  \texttt{lin.yang@princeton.edu}
  
}
\aistatsaddress{UT-Austin
\And Harvard \& UT-Austin
\And Princeton University}
]

\begin{abstract}
Parameter reduction has been an important topic in deep learning due to the ever-increasing size of deep neural network models and the need to train and run them on resource limited machines.
Despite many efforts in this area, there were no rigorous theoretical guarantees on why existing neural net compression methods should work.
In this paper, we provide provable guarantees on some hashing-based parameter reduction methods in neural nets.
First, we introduce a neural net compression scheme based on random linear sketching (which is usually implemented efficiently via hashing), and show that the sketched (smaller) network is able to approximate the original network on all input data coming from any smooth and well-conditioned low-dimensional manifold.
The sketched network can also be trained directly via back-propagation.
Next, we study the previously proposed HashedNets architecture and show that the optimization landscape of one-hidden-layer HashedNets has a local strong convexity property similar to a normal fully connected neural network.
We complement our theoretical results with empirical verifications. 

\end{abstract}

\section{Introduction}
%
In the past decade, deep neural networks have become the new standards for many machine learning applications, including computer vision \cite{krizhevsky2012imagenet,hzrs16}, natural language processing \cite{zsv14,gehring2017convs2s}, speech recognition \cite{gmh13,amodei2016deep}, robotics \cite{lhphetsw15}, game playing \cite{alphago16,alphago17}, etc.
Such model usually contains an enormous number of parameters, which is often much larger than the number of available training samples. 
Therefore, these networks are usually trained on modern computer clusters which have a huge amount of memory and computation power.
On the other hand, there is an increasing need to train and run personalized machine learning models on mobile and embedded devices instead of transferring mobile data to a remote computation center on which all the computations are performed.
This is because real-time processing of deep learning models on mobile devices brings the benefits of better privacy and less Internet bandwidth. 
However, mobile devices like smart phones do not have the memory or computation capability of training large neural networks or even storing these models. 

These trends motivate the study of \emph{neural network compression}, with the goal of reducing the memory overhead required to train, store and run neural networks.
There is a recent line of research in this direction, for example \cite{cwtwc15,ihmwadwk16,hmd16}.
Despite their empirical effectiveness, there is little theoretical understanding on why these methods perform well.

The goal of this paper is to bridge the gap between theory and practice in neural network compression.
Our focus is on \emph{hashing-based methods}, which have been studied empirically in e.g. \cite{cwtwc15,chen2016compressing}, with the hope that the randomness in hash functions helps preserve the properties of neural networks despite a reduction in the number of effective parameters.
We make this intuition formal by giving theoretical guarantees on the \emph{approximation power} and the \emph{parameter recovery} of such networks.

First, we propose a neural net compression scheme based on random linear sketching, which can be efficiently implemented using a hash function.
Similar idea has been proposed in \cite{knj17} and demonstrated high performance empirically, but no formal theoretical guarantee was known.
We show that such compression has strong approximation power. Namely, the small network obtained after sketching can approximate the original network on \emph{all} input data coming from \emph{any low-dimensional manifold} with some regularity properties.
The sketched network is also directly trainable via back-propagation.
In fact, sketching is a principled technique for dimensionality reduction, which has been shown to be very powerful in solving various problems arising in statistics \cite{rm16,wgm17} and numerical linear algebra \cite{w14}.
Given its theoretical success, it is natural to ask whether sketching can be applied to the context of neural net compression with theoretical guarantees. Our result makes partial progresses on this question.

Next we study HashedNets, a simple method proposed in \cite{cwtwc15} which appears to perform well in practice.
HashedNets directly applies a random hash function on the connection weights in a neural net and to enforce all the weights mapped to the same hash bucket to take the same value.
In this way the number of trainable parameters is reduced to be the number of different hash buckets, and training can still be performed via back-propagation while taking the weight sharing structure into account.
From the perspective of optimization, we show that the training objective for a one-hidden-layer hashed neural net has a local strong convexity property, similar to that of a normal fully connected network \cite{zsjbd17}.
Additionally, we can apply the initialization algorithm in \cite{zsjbd17} to obtain a good initialization for training.
Therefore it implies that the parameters in one-hidden-layer HashedNets can be provably recovered under milde assumptions.





Below we describe our contributions in more detail.
\paragraph{Approximation Power}
Our result on the approximation power of sketched nets is based on a classical concept, ``subspace embedding'', which originally appears in numerical linear algebra \cite{s06}. Roughly speaking, it says that there exist a wide family of random matrices $S \in \R^{s\times n}$, such that for \emph{any} $d$-dimensional subspace $\cU \subset \R^n$, with probability $1-\delta$ we have $\left\langle Sx, Sx' \right\rangle = \langle x, x' \rangle \pm \epsilon \|x\|_2  \|x'\|_2$ for all $x, x' \in \cU$, provided $s = \Omega\left(\left(d+\log 1/\delta\right)/\epsilon^2\right)$.
This result means that the inner product between every two points in a subspace can be approximated \emph{simultaneously} after applying a random \emph{sketching matrix} $S$, which is interesting if $s\ll n$.
There has been a line of work trying to do subspace embedding using different sketching matrices (e.g. \cite{nn13,c16}).
 \emph{Sparse matrices} are of particular interests, since for a sparse matrix $S$, one can compute $Sx$ more efficiently.
For example, \cite{nn13} showed that it is possible to construct $S$ with only $\tilde{O}(1/\epsilon)$ nonzero entries per column, which significantly improves the trivial upper bound $\tilde{O}(d/\epsilon^2)$.
Furthermore, many of these sketching matrices can be efficiently implemented by $k$-wise independent hash functions where $k$ is very small, which only takes a small amount of space to store, and multiplying $S$ with a vector can be computed efficiently.

We extend the idea of subspace embedding to deep learning and show that a feed-forward fully connected network with Lipschitz-continuous activation functions can be approximated using random sketching on all input data coming from a low-dimensional subspace.
Below we describe our result for one-hidden-layer neural nets, and this can be generalized to multiple layers.

Consider a one-hidden-layer neural net with input dimension $n$ and $k$ hidden nodes. It can be parameterized by a weight matrix $W \in \R^{n\times k}$ and a weight vector $v \in \R^k$, and the function this network computes is $x \mapsto v^\top \phi(W^\top x)$, where $x\in \R^n$ is the input, and $\phi$ should be viewed as a nonlinear activation function acting coordinate-wise on a vector.
Our result says that under appropriate assumptions, one can choose a random sketching matrix $S \in \R^{s\times n}$, such that for any $d$-dimensional subspace $\cU\subset \R^n$, we have
\begin{align*}
\left| v^\top \phi(W^\top x) - v^\top \phi(W^\top S^\top Sx)\right| \le \epsilon, \forall x \in \cU, \|x\|_2\le1.
\end{align*}
This result essentially says that the weight matrix $W^\top$ can be replaced by $W^\top S^\top S$, which has rank $s$.
When $s<n$, this means that the effective number of parameters can be reduced from $kn$ to $ks$.
As we mentioned, the sketching matrix $S$ can be implemented by hash functions in small space and multiplying it with a vector is efficient.
The sketched network is also directly trainable, because we can train the $s\times k$ matrix $\hat W = SW$, regarding another factor $S$ in the decomposition $W^\top S^\top S = \hat W S$ as a known layer.

This result can be generalized to multi-layer neural nets, and we present the details in Section~\ref{sec:neural_subspace_embedding}.
We also note that our result can be easily generalized to low-dimensional manifolds under some regularity condition (see Definition 2.3 in \cite{bw09}), which is a much more realistic assumption on data.

\paragraph{Parameter Recovery.}
It is known that training a neural net is NP-hard in the worst case, even if it only has $3$ hidden nodes \cite{br93}. Recently, there has been some theoretical progress on understanding the optimization landscapes of shallow neural nets under special input distributions.
In particular, \cite{zsjbd17} gave a recovery guarantee for one-hidden-layer neural nets. They showed that if the input distribution is Gaussian and the ground-truth weight vectors corresponding to hidden nodes are linearly independent, then the true parameters can be recovered in polynomial time given finite samples.
This was proved by showing that the training objective is locally strongly convex and smooth around the ground-truth point, together with an initialization method that can output a point inside the locally ``nice'' region.
In this work, we show that local strong convexity and smoothness continue to hold if we replace the fully connected network by HashedNets which has a weight sharing structure enforced by a hash function.
We present this result in Section~\ref{sec:recovery}.

\subsection{Related Works}
\paragraph{Parameter Reduction in Deep Learning.}
There has been a series of empirical works on reducing the number of free parameters in deep neural networks:
\cite{denil2013predicting} show a method to learn low-rank decompositions of weight matrices in each layer,
\cite{cwtwc15} propose an approach to use a hash function to enforce parameter sharing,
\cite{cheng2015exploration} adopt a circulant matrix structure for parameter reduction,
\cite{sindhwani2015structured} study a more general class of structured matrices for parameter reduction.
\paragraph{Sketching and Neural Networks.}
\cite{dlst16} show that any linear or sparse polynomial function on sparse binary data can be computed by a small single-layer neural net on a linear sketch of the data.
\cite{knj17} apply a random sketching on weight matrices/tensors, but they only prove that given a fixed layer input, the output of this layer using sketching matrices is an unbiased estimator of the original output of this layer and has bounded variance; however, this does not provide guarantees on the approximation power of the whole sketching-based deep net.
\paragraph{Subspace Embedding.}
Subspace embedding \cite{s06} is a fundamental tool for solving numerical linear algebra problems, e.g. linear regression, matrix low-rank approximation \cite{cw13,nn13,rsw16,swz17}, tensor low-rank approximation \cite{swz19}.
See also \cite{w14} for a survey on this topic.
\paragraph{Recovery Guarantee of Neural Networks.}
Since learning a neural net is NP-hard in the worst case \cite{br93}, many attempts have been made to design algorithms that learns a neural net provably in polynomial time and sample complexity under additional assumptions, e.g., \cite{sa14, zlwj15, jsa15, gkkt17, gk17a, gk17b}.
Another line of work focused on analyzing (stochastic) gradient descent on shallow networks for Gaussian input distributions, e.g., \cite{bg17, zsd17, zsjbd17, tian2017analytical, ly17, dltps17, soltanolkotabi2017learning}.

\paragraph{Other Related Works}
Instead of understanding the parameter reduction as our work, there are several results working on developing over-parameterization theory of deep ReLU neural networks, e.g. \cite{als18a,als18b}. Thirty years ago, Blum and Rivest proved training neural network is NP-hard \cite{br93}. Later, neural networks have been shown hard in several different perspectives \cite{ks09,lss14,d16,ds16,gkkt17,svwx17,kbdjk17,wzcshbdd18,mr18} in the worst case regime. 

Arora et al. proved a stronger generalization for deep nets via a compression approach \cite{agnz18}. There is a long line of works targeting on explaining GAN from theoretical perspective \cite{az17,arz17,aglmz17,bjpd17,lmps2018,ssm18,vjpvd18,xzz18}. There is also a long line of provable results about adversarial examples \cite{mmstv18,bpr18,blpr18,wzcshbdd18,ssttm18,tlm18}.

\section{Preliminaries}
For any positive integer $n$, we use $[n]$ to denote the set $\{1,2,\cdots,n\}$.
Let $a\pm b$ represent any number in the interval $[a-b, a+b]$.
For any vector $ x\in \mathbb{R}^n$, we use $\|x\|_2$, $\| x\|_1$ and $\| x \|_{\infty}$ to denote its $\ell_2$, $\ell_1$ and $\ell_{\infty}$ norms, respectively.
For $x, y \in \R^n$,
we use $\langle x, y \rangle$ to denote the standard Euclidean inner product $x^\top y$.


For a matrix $A$,
let $\det(A)$ denote its determinant (if $A$ is a square matrix), 
let $A^\dagger$ denote the Moore-Penrose pseudoinverse of $A$, 
and let $\| A\|_F$ and $\| A\|_2$ denote respectively the Frobenius norm and the spectral norm of $A$.
Denote by $\sigma_i(A)$ the $i$-th largest singular value of $A$. We use $\nnz(A)$ to denote the number of non-zero entries in $A$.

For any function $f$, we define $\widetilde{O}(f)$ to be $f\cdot \log^{O(1)}(f)$. In addition to $O(\cdot)$ notation, for two functions $f,g$, we use the shorthand $f\lesssim g$ (resp. $\gtrsim$) to indicate that $f\leq C g$ (resp. $\geq$) for an absolute constant $C$. We use $f\eqsim g$ to mean $cf\leq g\leq Cf$ for constants $c,C$.

We define the $\ell_1$ and $\ell_2$ balls in $\R^n$ as: 
	$ 
	{\cal B}_1(B,n) =  \{ x \in \R^n ~|~ \| x \|_1 \leq B \}, 
	{\cal B}_2(B,n) = \{ x \in \R^n ~|~ \| x \|_2 \leq B \}.
	$
We also need the definitions of Lipschitz-continuous functions and  $k$-wise independent hash families.

\begin{definition}\label{def:lipshitz}
	A function $f : \R \rightarrow \R$ is $L$-Lipshitz continuous, if for all $x_1, x_2 \in \R$,
	$
	| f( x_1 ) - f( x_2 ) | \leq L | x_1 - x_2 |.
	$
\end{definition}

\begin{definition} 
A family of hash functions ${\cal H} \subseteq \{ h \mid h: U \to [B] \}$ is said to be {$k$-wise independent} if for any $x_1, \ldots, x_k \in U$ and any
$y_1, \ldots, y_k \in [B]$ we have
$
    \Pr_{h \sim {\cal H}} [h(x_i) = y_i, \forall i\in[k]] = \frac{1}{B^{k}}.
$
\end{definition}

\section{
The Approximation Power of Parameter-Reduced Neural Networks
}\label{sec:neural_subspace_embedding}
In this section, we study 
the approximation power of parameter-reduced neural nets based on hashing.
Any weight matrix $W$ in a neural net acts on a vector $x$ as $Wx$. We replace $Wx$ by $WS^\top S x$ for some sketching ($\#$rows $<$ $\#$columns) matrix $S$ defined in the following section.
Then the new weight matrix $W S^\top$ has much fewer parameters.  
We show that if $S$ is chosen properly as a \emph{subspace embedding} (formally defined later in this section), the sketched network can approximate the original network on all inputs coming from a low-dimensional subspace or manifold. 
Our sketching matrix is chosen as a Johnson-Lindenstrauss (JL) \cite{jl84} transformation matrix. 
In Section~\ref{sec:jl_transform_subspace_embedding}, we provide some preliminaries on subspace embedding.
In Sections~\ref{sec:one_hidden_1} and \ref{sec:one_hidden_2}, we present our result on one-hidden-layer neural nets.
Then in Section~\ref{sec:multiple_hidden} we extend this result to multi-layer neural nets and show a similar approximation guarantee.
This provides a theoretical guarantee for hashing-based parameter-reduced networks used in practice.
\subsection{Subspace Embedding}\label{sec:jl_transform_subspace_embedding}
We first present some basic definitions of sketching and subspace embedding. 
These mathematical tools are building blocks for us to understand parameter-reduced neural networks.
\begin{definition}[Subspace Embedding] 
	\label{def:subspace}
	A $(1\pm \epsilon)$ $\ell_2$-subspace embedding for the column space of an $n \times d$ matrix $U$ is a matrix $S$ for which for all $x \in \mathrm{colspan}(U)$,
	$
	\| S x \|_2^2 = ( 1 \pm \epsilon ) \| x \|_2^2,
	$
	or equivalently, for all $x, x' \in \mathrm{colspan}(U)$,
	$
	\langle Sx, Sx' \rangle = \langle x, x' \rangle \pm \epsilon \|x\|_2\|x'\|_2.
	$
\end{definition}
Constructions of subspace embedding can be found in e.g. \cite{nn13} from which there is the following theorem.
\begin{theorem}[\cite{nn13}]\label{thm:sparse_embedding_subspace}
	There is a $(1\pm \epsilon)$ oblivious\footnote{The construction of $S$ is oblivious to the subspace $U$.} $\ell_2$-subspace embedding for $n \times d$ matrix $U$ with $s = d \cdot \poly\log( d / (\epsilon \delta) ) /\epsilon^2$ rows and error probability $\delta$. Further, $S \cdot U$ can be computed in time $O(\nnz(U) \poly\log ( d / (\epsilon \delta) ) /\epsilon)$. We call $S$ a $\textsc{SparseEmbedding}$ matrix. 
\end{theorem}
There are also other subspace embedding matrices, e.g., \textsc{CountSketch}.
We provide additional definitions and examples in Section~\ref{sec:nural_subspace_embdding_defs}.
\begin{remark}
We remark that the subspace embedding in Definition~\ref{def:subspace} naturally extends to low dimensional manifolds. 
For example, for a $d$-dimensional Riemannian submanifold of $\R^{n}$ with volumn $V$ and geodesic covering regularity $R$ (see Definition 2.3 in \cite{bw09}), Theorem~\ref{thm:sparse_embedding_subspace} holds by replacing $d$ with $Rd\log(V)$.
For ease of presentation, we only present our results for subspaces. 
All our results can be extended to low-dimensional manifolds satisfying regularity conditions. 
\end{remark}

\subsection{One Hidden Layer - Part \RN{1}}\label{sec:one_hidden_1}

We consider one-hidden-layer neural nets in the form $\sum_i v_i \phi_i(w_i^\top x)$, where $x$ is the input vector, $w_i$ is a weight vector, $v_i$ is a weight scalar, and $\phi_i: \R\to\R$ is a nonlinear activation function.
In this subsection, we show how to sketch the weights between the input layer and the hidden layer with guaranteed approximation power. The main result is Theorem~\ref{thm:subspace_one_hidden_I_formal} and its proofs are in Appendix~\ref{sec:app_subspace}.

\begin{theorem}
\label{thm:subspace_one_hidden_I_formal}
Given parameters $n_2 \geq 1, n_1\geq 1 ,\epsilon \in (0,1), \delta \in (0,1)$ and $n=\max(n_2,n_1)$.
 Given $n_2$ activation functions $\phi_i :\R \rightarrow \R$ that are $L$-Lipshitz-continuous,
 a fixed matrix $U \in \R^{n_1 \times d}$, weight matrix $W \in \R^{n_1 \times n_2}$ with $\| W \|_2 \leq B$, $v \in \R^{n_2}$ with $\| v\|_1 \leq B$. Choose a \textsc{SparseEmbedding} matrix $S \in \R^{s \times n_1}$ with $s = O( d L^2 B^4 A^2 \epsilon^{-2} \poly \log (n L B A / (\epsilon \delta) )  ) $, then with probability $1-\delta$, we have : for all $x \in \mathrm{colspan}(U) \cap {\cal B}_2(A, n_1)$, 
\begin{align*}
| \langle v, \phi(W^\top x) \rangle - \langle v, \phi(W^\top S^\top S x ) \rangle | \leq \epsilon.
\end{align*}
\end{theorem}

\subsection{One-hidden layer - Part \RN{2}}\label{sec:one_hidden_2}
In this section, we show the approximation power of the compressed network if the weight matrices of both the input layer and output layer are sketched. One of the core idea in the proof is a recursive $\epsilon$-net argument, which plays a crucial role in extending the result to multiple hidden layer. The goal of this section is to prove the following theorem and present the recursive $\epsilon$-net argument.
\begin{theorem}\label{thm:subspace_one_hidden_II}
Given parameters $n_2 \geq 1, n_2\geq 1 ,\epsilon \in (0,1), \delta \in (0,1)$ and $n=\max(n_2,n_1)$.
 Given $n_2$ activation functions $\phi_i :\R \rightarrow \R$ with $L$-Lipshitz and normalized by $1/\sqrt{n_2}$, a fixed matrix $U \in \R^{n_1 \times d}$, and weight matrix $W \in \R^{n_2 \times n_1}$ with $\| W \|_2 \leq B$,
$v \in \R^{n_2}$ with $\| v\|_2 \leq B$. Choose a \textsc{SparseEmbedding} matrix $S_1 \in \R^{s_1 \times n_1}$ and $S_2 \in \R^{s_2 \times n_2}$ with $s_1,s_2 = O( d  L^2 B^4 A^2 \epsilon^{-2} \poly \log (n L B A / (\epsilon \delta) )  ) $, then with probability $1-\delta$, we have : for all $x \in \mathrm{colspan}(U) \cap {\cal B}_2(A, n_1)$, 
\begin{align*}
 \left| \langle v, \phi(W^\top x) \rangle - \langle S_2 v, S_2 \phi(W^\top S_1^\top S_1 x) \rangle \right|
\leq  \epsilon  .
\end{align*}
\end{theorem}

The high level idea is as follows. 
Firstly, we prove that for any fixed input $x$, the theorem statement holds with high probability. 
Then we build an sufficiently fine $\epsilon$-net over the input space of $x$ and argue that our statement holds for every input point $x$ from the $\epsilon$-net. 
Condition on this event, the statement holds by applying the Lipshitz continuity of the activation function.
The detailed proof is presented in Appendix~\ref{sec:app_subspace}.

\subsection{Multiple hidden layer}\label{sec:multiple_hidden}
In this section, we generalize our approximation power result to a multi-layer neural network and delay the proofs to Appendix~\ref{sec:app_subspace}. Inspired by the batch normalization \cite{is15}, which has been widely used in practice\footnote{\url{https://www.tensorflow.org/api_docs/python/tf/nn/batch_normalization}}, we make an additional assumption by requiring the activations to be normalized by $1/\sqrt{n_{j+1}}$ at each layer $j$. The way we deal with multiple hidden layers is, first recursively argue an $\epsilon$-net can be constructed for all the layers with the same size.
Then we use triangle inequality to split error into $q+1$ terms and bounding them separately. The result is the following theorem. 
\begin{theorem}[]\label{thm:subspace_multiple_hidden_informal}
Given parameters $q\geq 1$, $n_j \geq 1,\forall j \in [q+1] ,\epsilon \in (0,1), \delta \in (0,1)$ and $n=\max_{j \in [q+1]} n_j$. For each $j \in [q]$, for each $i \in [n_{j+1}]$
let $\phi_{j,i} :\R \rightarrow \R$ denote an activation function  with $L$-Lipshitz and normalized by $1/\sqrt{n_{j+1}}$.
Given a fixed matrix $U \in \R^{n_1 \times d}$, $q$ weight matrices $W_j \in \R^{n_{j+1} \times n_j}, \forall j \in [q]$ with (the $i_j$-th column of $W_j$) $w_{j,i_j} \in {\cal B}_2(B,n_j)$, a weight vector $v\in \R^{n_{j+1}}$ with $v \in {\cal B}_{2}(B,n_{j+1})$. For each $j \in [q+1]$, we choose a \textsc{SparseEmbedding} matrix $S_j \in \R^{s_j \times n_j}$ with 
 \begin{align*}
 s_j= O( d q^2 L^{2q} B^{2q+2} A^2 \epsilon^{-2} \poly \log (n q L B A / (\epsilon \delta) )  ). 
 \end{align*}
Then with probability $1-\delta$, we have : for all $x \in \mathrm{colspan}(U) \cap {\cal B}_2(A, n_1)$, 
\begin{align*}
| \langle v, f^{(q)}(x) \rangle - \langle S_{q+1} v, S_{q+1} \widetilde{f}^{(q)}(x) \rangle | \leq  \epsilon,
\end{align*}
where $f^{(q)}(x)$ and $\wt{f}^{(q)}(x)$ are defined inductively. The base case is $f^{(0)}(x) =  \widetilde{f}^{(0)}(x) = x,$ and the inductive case is
\begin{align*}
    f^{(j)}(x) &=  \phi_j(W_j^\top f^{(j-1)}(x) ), \forall j \in [q] \\
    \mathrm{~and~}
    \widetilde{f}^{(j)}(x) & =  \phi_j(W_j^\top S_j^\top S_j \widetilde{f}^{(j-1)}(x) ), \forall j \in [q].
\end{align*}
\end{theorem}
Note that similar results also hold for the case without using $S_{q+1}$. In other words, we only choose $q$ matrices for $q$ hidden layers.

\section{Recovery Guarantee}\label{sec:recovery}
In this section, we study the recovery guarantee of parameter-reduced neural nets. In particular, we study whether (stochastic) gradient descent can learn the true parameters in a one-hidden-layer HashedNets when starting from a sufficiently good initialization point, under appropriate assumptions.
We show that even under the special weight sharing structure depicted by the hash function, the resulting neural net still has sufficiently nice properties - namely, local strong convexity and smoothness around the minimizer.
Our proof technique is by reducing our case to that of the fully connected network studied in \cite{zsjbd17}.
After that, the recovery guarantee follows similarly. We present our result here and give the detailed proof in Appendix \ref{app:recovery}.

We consider the following regression problem : given a set of $m$ samples
$
S = \{ (x_1,y_1), (x_2,y_2), \cdots, (x_m,y_m) \} \subset \R^n \times \R.
$
Let ${\cal D}$ be an underlying distribution over $\R^n \times \R$ with parameters $w^* \in \R^B$
 and $v^* \in \R^k$, such that each sample $(x,y) \in S$ is sampled i.i.d. from this distribution, with
$
x \sim {\cal N}(0,I), y = \sum_{i=1}^k v_i^* \cdot \phi ( \sum_{j=1}^n w_{h(i,j)}^* \cdot x_j ).
$
Here $h : [k] \times [n] \rightarrow [B]$ is a random hash function drawn from a $t$-wise independent hash family ${\cal H}$, where $t = \Theta(\log(nk))$, and $\phi$ is an activation function.

Note that $w^*$ has a corresponding matrix $\hat W^* \in \R^{k\times n}$ defined as $\hat W^*_{ij} = w^*_{h(i, j)}$, which is the actual weight matrix in the HashedNets with a weight sharing structure. 


Our goal is to recover the ground-truth parameters $w^*, v^*$ given the sample $S$. Note that how to recover $v^*$ has been discussed in \cite{zsd17, zsjbd17}; their method also applies to our situation. Therefore we focus on recovering $w^*$ in this section, assuming $v^*$ is known.

For a given weight vector $w\in \R^B$,
we define its {\it expected risk} and {\it empirical risk}  as {\small
\begin{align*}
    F_{\cal D}(w) &= \frac{1}{2} \E_{(x,y) \sim {\cal D}} \left[ \left( \sum_{i=1}^k v_i^* \cdot \phi \left( \sum_{j=1}^n w_{h(i,j) } \cdot x_j \right) - y \right)^2 \right] ,  \\
    F_{S}(w) &= \frac{1}{2} \sum_{(x,y) \sim S} \left[ \left( \sum_{i=1}^k v_i^* \cdot \phi \left( \sum_{j=1}^n w_{h(i,j) } \cdot x_j \right) - y \right)^2 \right].
\end{align*}}

We first show a structural result for $t$-wise independent hash family, which says the pre-image of  each bucket is pretty balanced.
\begin{lemma}[Concentration of hashing buckets, part of Lemma~\ref{lem:concentration_of_hashing_buckets}]\label{lem:concentration_of_hashing_buckets_informal}
Given integers $N$ and $B \lesssim N / \log N$. Let $h : [N] \rightarrow [B]$ denote a $t$-wise independent hash function such that 
$
\Pr_{h \sim {\cal H} }[ h(i) = j] = 1/B, \forall i \in [N], \forall j \in [B].
$
Then, if $t = \Theta(\log N)$, with probability at least $1 - 1 / \poly(N)$, we have for all $j \in [B]$,
$
0.9 N/B \leq \sum_{i=1}^N {\bf 1}_{h(i) = j} \leq 1.1 N/B.
$
\end{lemma}

The previous work \cite{zsjbd17} showed that a fully connected network whose ground-truth weight matrix $W^* \in \R^{k\times n}$ has rank $k$ has local strong convexity and smoothness around $W^*$ in its loss function (see their Lemma D.3).


Using Lemma~\ref{lem:concentration_of_hashing_buckets_informal} as our core reduction tool, we can reduce HashedNets to a fully connected net and obtain the following result:
\begin{theorem}[Local strong convexity and smoothness] \label{thm:hashnet-local-sc}
Suppose $\rank( \hat W^* ) = k$.
Then we have \begin{align*}
 0.5 ( kn / B ) \cdot A_{\min} \cdot I \preceq  \nabla^2 F_{\cal D}(w^*) \preceq 2 ( kn / B ) \cdot A_{\max} \cdot  I,
\end{align*}
where $A_{\max}$ and $A_{\min}$ are positive parameters that depend on $\hat W^*$ and the activation function $\phi$.
\end{theorem}

\begin{remark}
	A crucial assumption in Theorem~\ref{thm:hashnet-local-sc} is that the weight matrix $\hat W^*$ has rank $k$.
	In Section~\ref{sec:exp}, we use numerical experiment to verify this assumption in learned HashedNets.
\end{remark}

For the empirical risk $F_S$, we can show that its Hessian $\nabla^2 F_{S}(w^*)$ at the optimal point also satisfies similar properties given enough samples. See Theorem~\ref{thm:hashnet_strong_convex_S_formal} for details.


Using the tensor initialization method in \cite{zsjbd17}, we can find a point in the locally ``nice region'' around $w^*$, and
then we can show that gradient descent on the empirical risk function $F_S$ converges linearly to $w^*$.
The result is summarized as follows.
\begin{theorem} [Recovery guarantee]\label{thm:recovery_guarantee} 
There exist parameters $\gamma_1$ and $\gamma_2$ that depend on $\hat W^*$ and $\phi$ such that the following is true.
Let $w^c$ be any point satisfying 
$
\| w^c - w^* \|_2 \le \gamma_1 \|w^*\|_2,
$
and let $S$ denote a set of i.i.d. samples from the distribution ${\cal D}$. 
Define
$
m_0 = \Theta( \frac{kn}{B}A_{\min} ) \text{~and~} M_0 = \Theta(\frac{kn}{B} A_{\max})
$
where $A_{\max}$ and $A_{\min}$ are the same ones in Theorem~\ref{thm:hashnet-local-sc}. 
For any $t\geq 1$, if we choose
$
|S| \geq n \cdot \poly (\log n, t) \cdot k^2 \gamma_2 
$
and perform gradient descent with step size $1/M_0$ on $F_S$ and obtain the next iterate,
$
\wt{w} = w^c - \frac{1}{M_0} \nabla F_S(w^c),
$
then with probability at least $1-n^{-\Omega(t)}$, we have
$
\| \wt{w} - w^* \|_2^2 \leq ( 1 - m_0 / M_0 ) \| w^c - w^* \|_2^2.
$
\end{theorem}
The above theorem states that once a constantly-accurate initialization point is specified, we can obtain a solution up to  precision $\exp(-\poly(n))$ in a polynomial number of gradient descent iterations.  This concludes the recovery guarantee. 
We give the formal statements and proofs in Section~\ref{app:recovery}. 

\section{Experiments}\label{sec:exp}
In this section, we perform some simple experiments on MNIST dataset to evaluate the performance of HashedNets, as well as empirically verify the full rank assumption (as in Theorem~\ref{thm:hashnet-local-sc}) on weight matrices in HashedNets.
Each image in MNIST dataset has a dimensionality of $28 \times 28$. 
The HashedNets in the experiment have single-hidden-layer, i.e., two fully connected layers. 
To validate the effectiveness of HashedNets, we construct two baselines. 
\begin{itemize}
    \item \textbf{SmallNets}. A single-hidden-layer network is constructed with the same amount of effective weights as that of HashedNets. 
        For example, for a HashedNets with $1000$ hidden units in the hidden layer with compression ratio 64, a corresponding SmallNets have $\lceil {\frac{1000}{64}} \rceil =16$ hidden units in the hidden layer. 
    \item \textbf{ThinNets}. A two-hidden-layer network is constructed with the same amount of effective weights as that of HashedNets. 
        By replacing the first fully connected layer in HashedNets with a thin hidden layer, a same amount of weights can be achieved. 
        For example, for a HashedNets with $1000$ hidden units in the hidden layer with compression ratio 64, a corresponding ThinNets have $\lceil {\frac{784 \times 1000}{(784+1000) \times 64}} \rceil =7$ hidden units for the first hidden layer and 1000 hidden units for the second hidden layer. 
\end{itemize}
The accuracy of HashedNets, SmallNets, and ThinNets is compared under various compression ratios. 

\begin{figure*}[t]
\centering
{\includegraphics[width=0.24\textwidth]{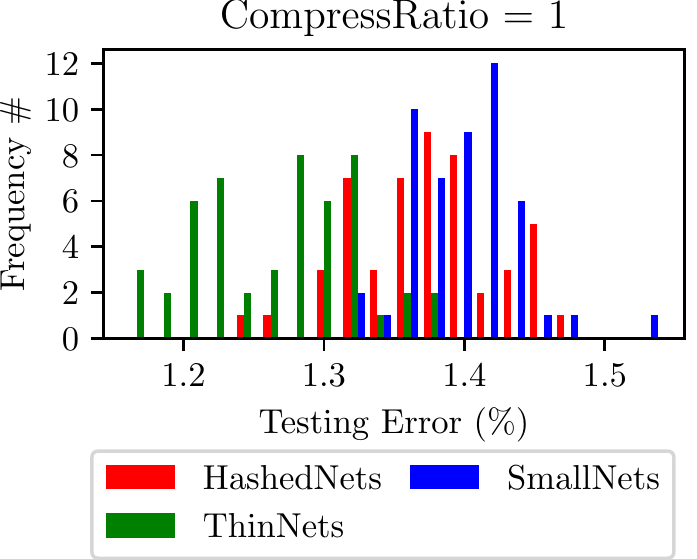}}\label{fig:testAccuracyHistogramCR1}
{\includegraphics[width=0.24\textwidth]{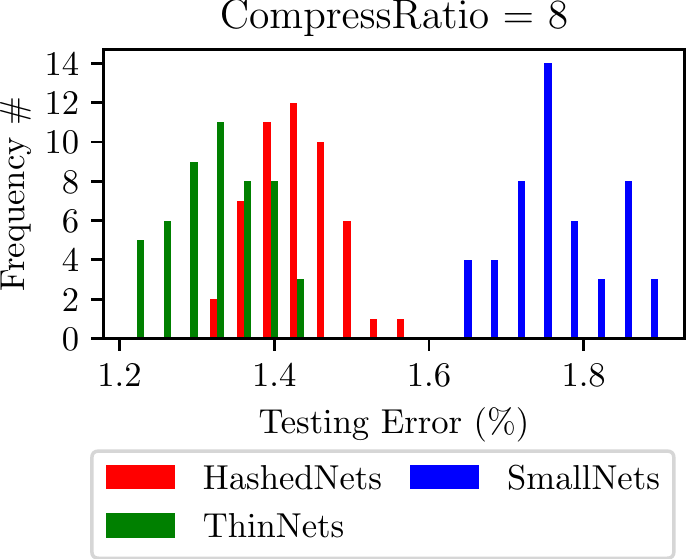}}\label{fig:testAccuracyHistogramCR8}
{\includegraphics[width=0.24\textwidth]{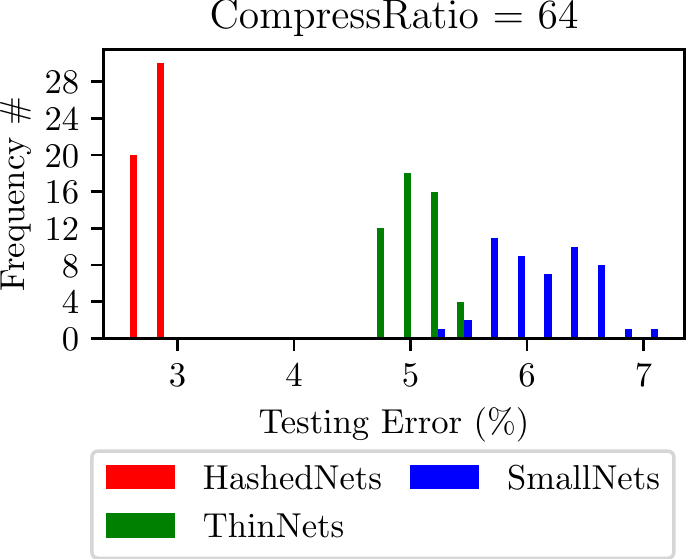}}\label{fig:testAccuracyHistogramCR64}
{\includegraphics[width=0.24\textwidth]{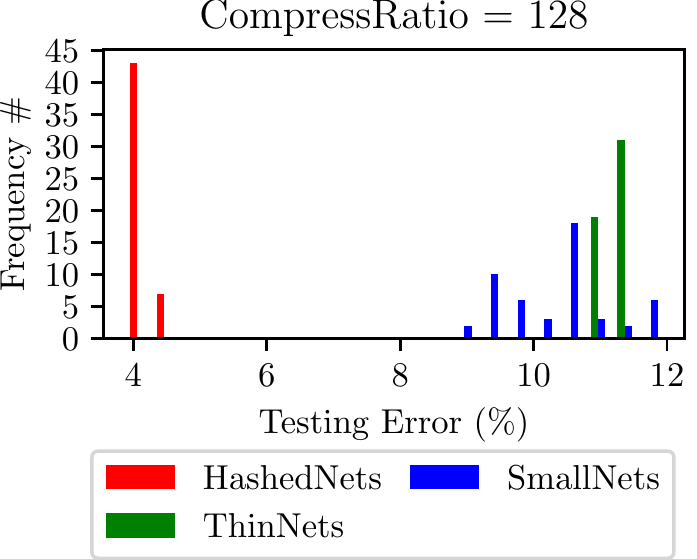}}\label{fig:testAccuracyHistogramCR128}
\caption{\small (a) Ratio$=1$; (b) ratio$=8$; (c) ratio$=64$; (d) ratio$=128$. We run two one-hidden layer algorithms on MNIST dataset. Comparison of accuracy distribution with different random seeds for HashedNets and SmallNets. Choose 50 random seeds in total. HashedNets have 1000 hidden units in this case. }
\label{fig:testAccuracyHistogram}
\end{figure*}

\begin{figure*}[t]
\centering
{\includegraphics[height=0.23\textwidth]{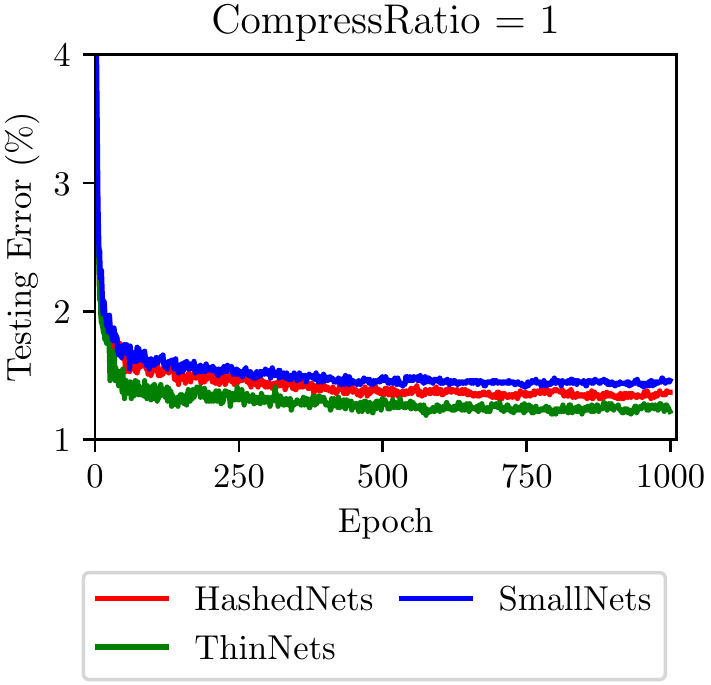}}\label{fig:testAccuracyCurveCR1}
{\includegraphics[height=0.23\textwidth]{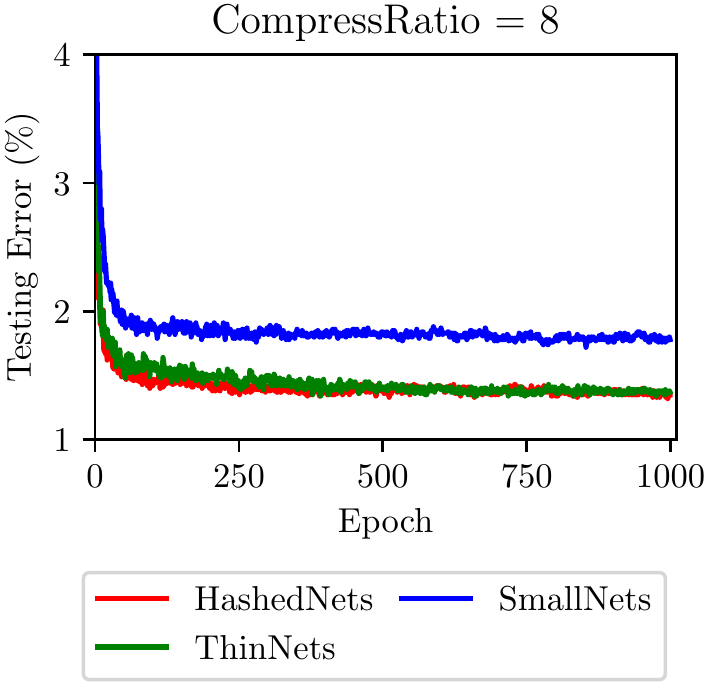}}\label{fig:testAccuracyCurveCR8}
{\includegraphics[height=0.23\textwidth]{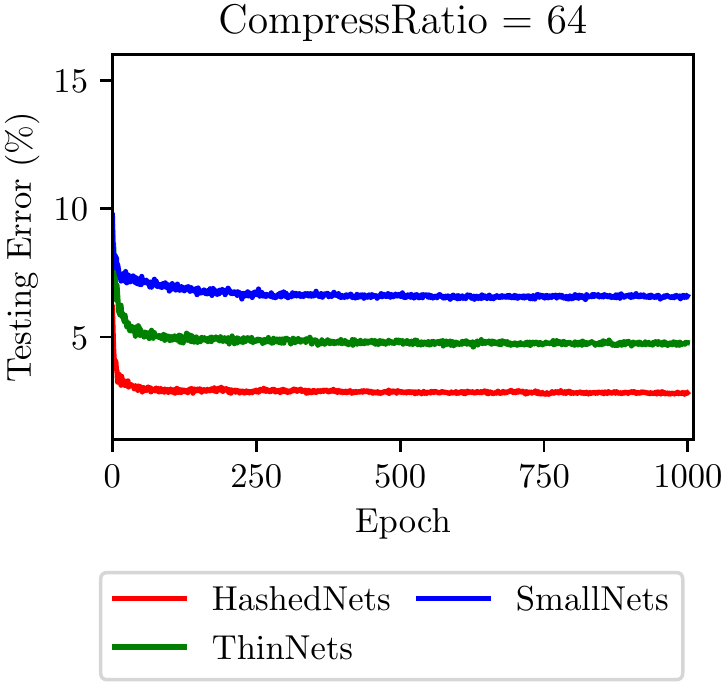}}\label{fig:testAccuracyCurveCR64}
{\includegraphics[height=0.23\textwidth]{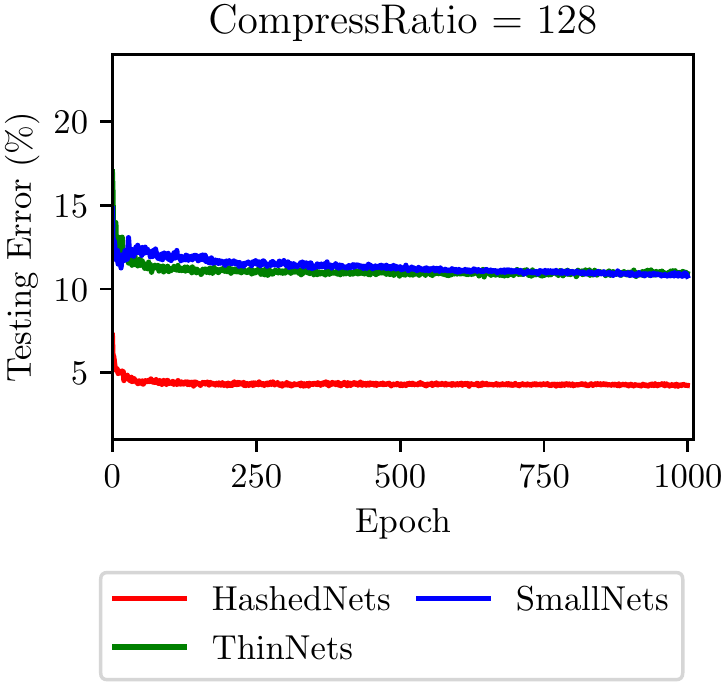}}\label{fig:testAccuracyCurveCR128}
\caption{\small (a) Ratio$=1$; (b) ratio$=8$; (c) ratio$=64$; (d) ratio$=128$. The testing error during training of different networks with random seed 100.}
\label{fig:testAccuracyCurve}
\end{figure*}




The HashedNets were implemented in Torch7 \cite{ckf11, cwtwc15} and validated on NVIDIA GTX 1080 GPU. 
We used 32 bit precision floating point numbers throughout the experiments. 
Stochastic gradient descent was adopted as the numerical optimizer with a dropout keep rate of 0.9, momentum of 0.9, and a batch size of 50. 
ReLU was used as the activation function. 
We ran 1000 epochs for each experiment and experiments on two single-hidden-layer HashedNets with 500 and 1000 hidden units are conducted, respectively. 
The amount of units in SmallNets and ThinNets is adjusted to match the amount of weights in HashedNets. 



\begin{figure*}[!t]
\centering
{\includegraphics[width=0.25\textwidth]{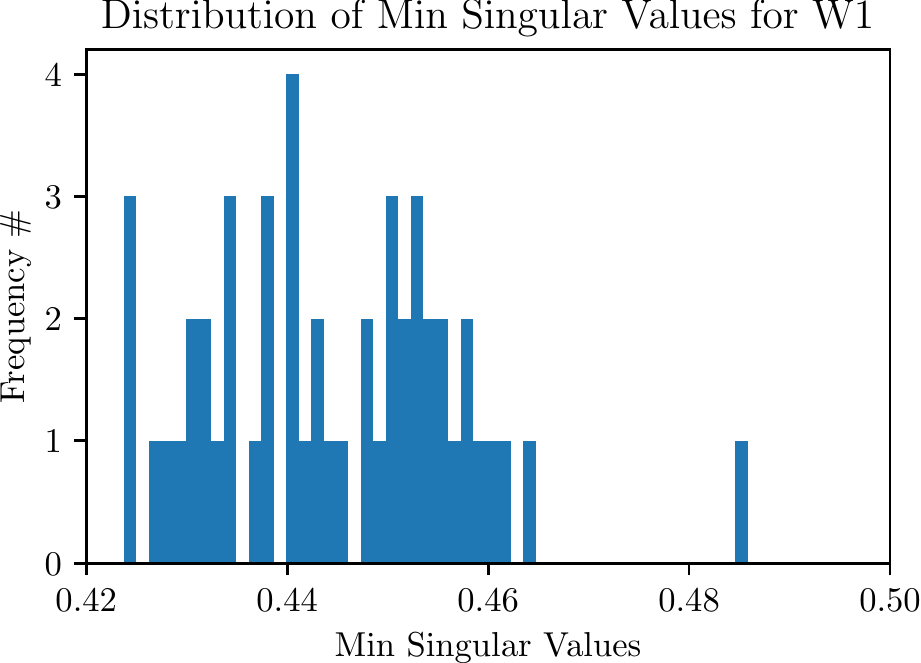}}\label{fig:Min_Singular_Values_for_W1_CR64}%
{\includegraphics[width=0.25\textwidth]{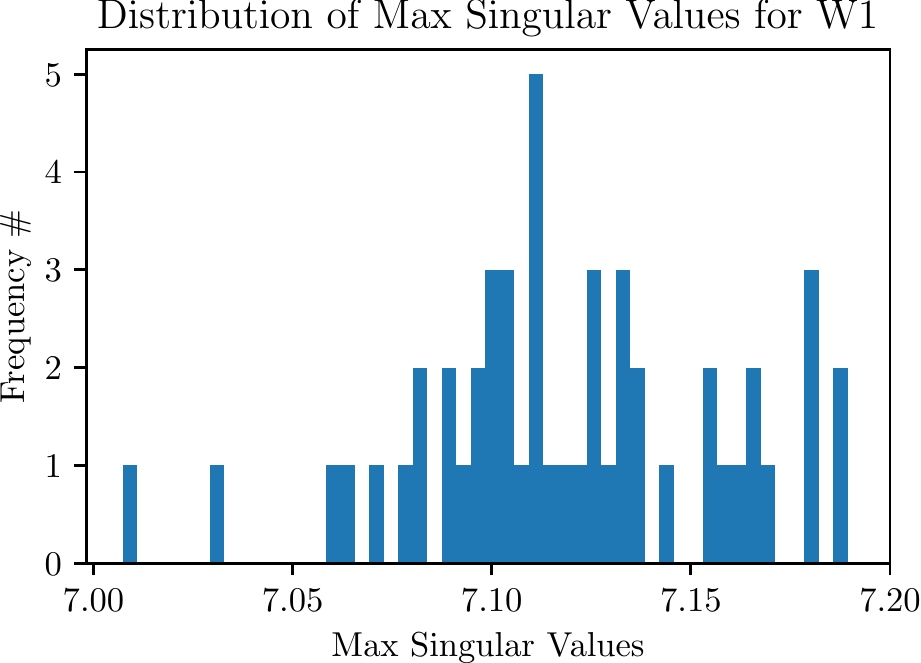}}\label{fig:Max_Singular_Values_for_W1_CR64}%
{\includegraphics[width=0.25\textwidth]{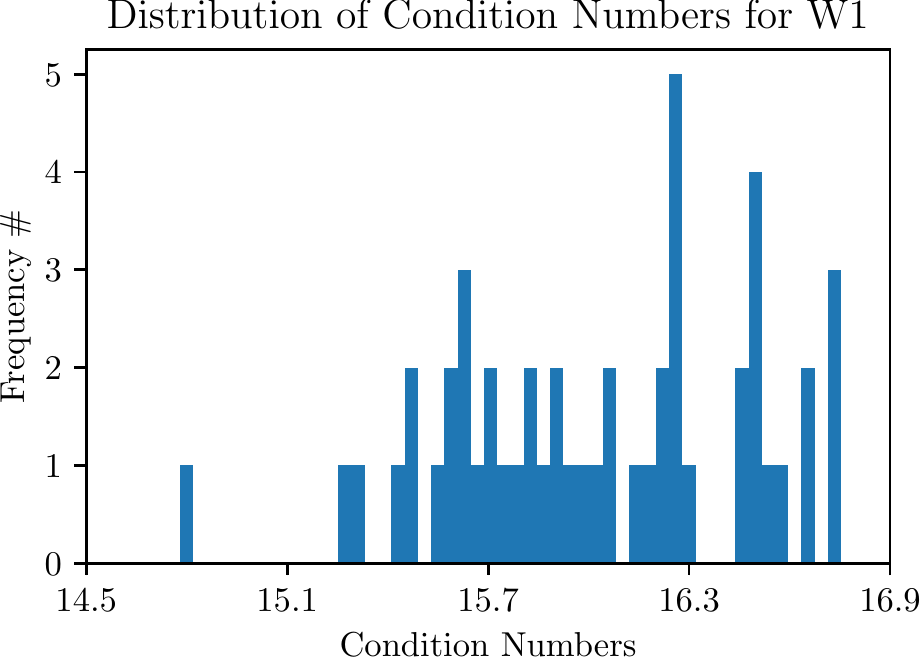}}\label{fig:Condition_Numbers_for_W1_CR64}%
{\includegraphics[width=0.25\textwidth]{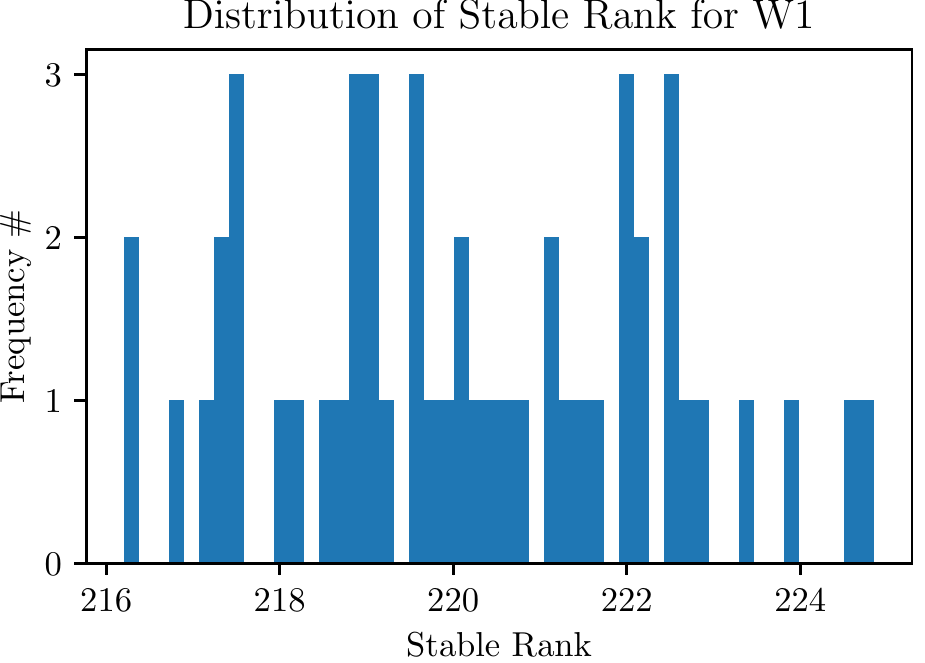}}\label{fig:Stable_Rank_for_W1_CR64}\\
\vspace{.1in}
{\includegraphics[width=0.25\textwidth]{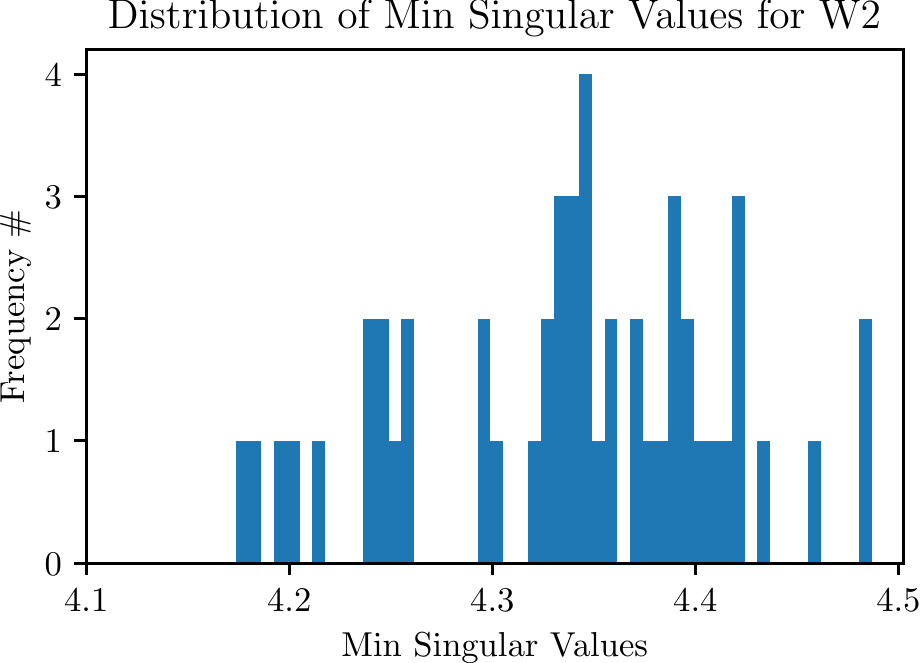}}\label{fig:Min_Singular_Values_for_W2_CR64}%
{\includegraphics[width=0.25\textwidth]{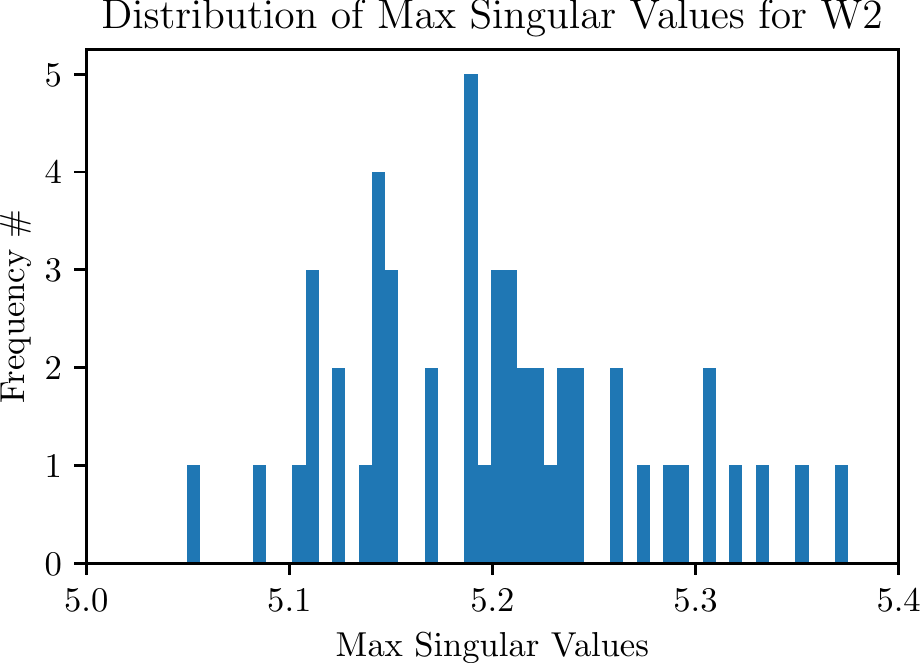}}\label{fig:Max_Singular_Values_for_W2_CR64}%
{\includegraphics[width=0.25\textwidth]{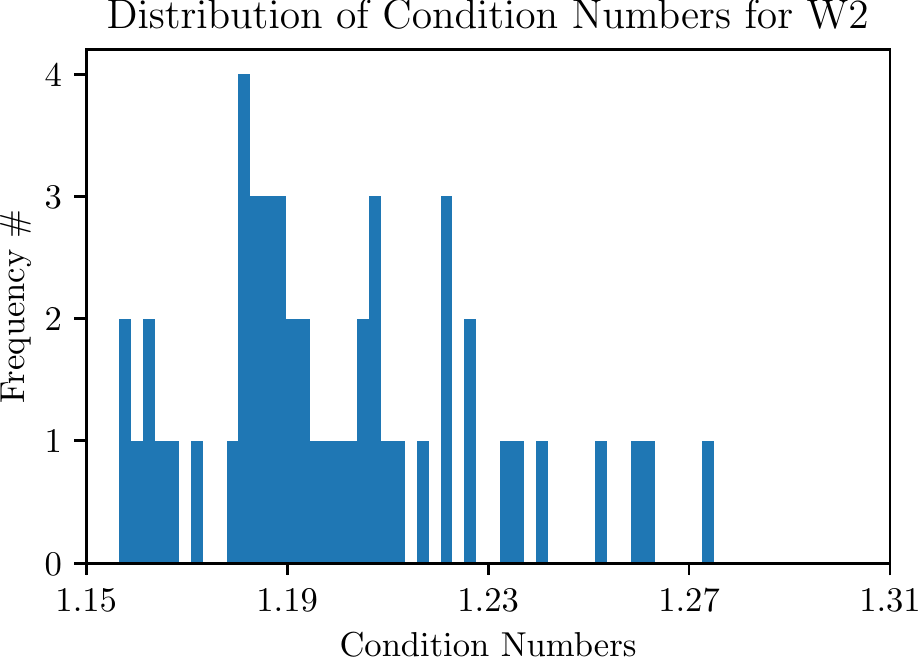}}\label{fig:Condition_Numbers_for_W2_CR64}%
{\includegraphics[width=0.25\textwidth]{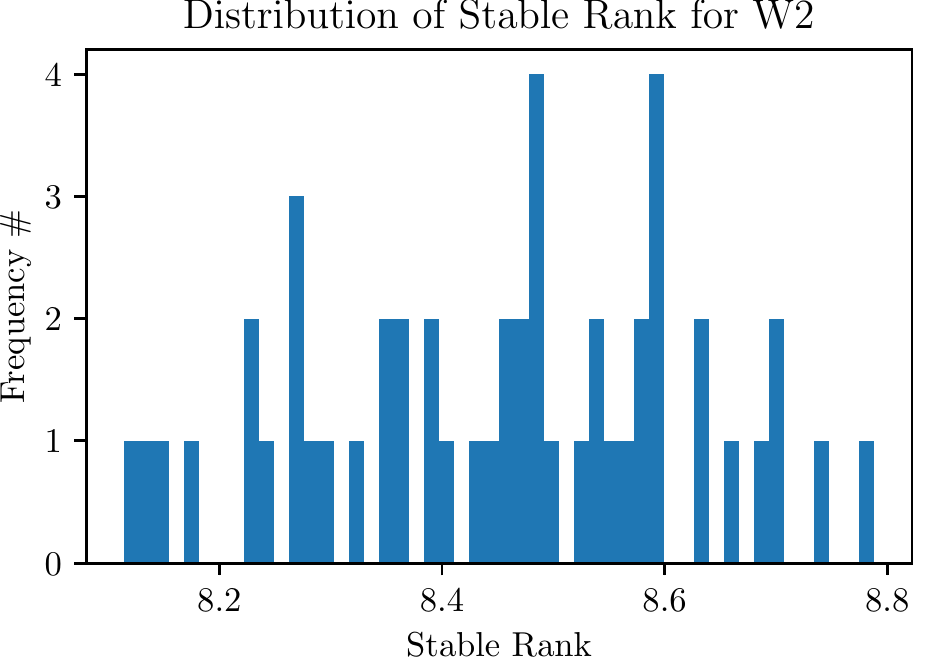}}\label{fig:Stable_Rank_for_W2_CR64}
\caption{\small Input dimension is 784. Distributions of singular values, condition numbers, and stable ranks for two weight matrices $W_1$ and $W_2$ in HashedNets with 1000 hidden units for 50 random seeds.}
\label{fig:ConditionNumbers}
\end{figure*}

For different compression ratios, we plot the distribution of testing errors for 50 runs of HashedNets, ThinNets, and SmallNets with 50 different random seeds, as shown in Figure~\ref{fig:testAccuracyHistogram}. 
Due to random initialization, SmallNets still gives different results with independent runs. 
In Figure~\ref{fig:testAccuracyHistogram}(a), when the compression ratio is 1, which indicating no compression, the distributions for both HashedNets and SmallNets are very close, 
i.e., with means of 1.37\% and 1.40\%, standard deviations of 0.050\% and 0.038\%, respectively. 
ThinNets provides slightly better testing error with a mean of 1.27\% and a standard deviation of 0.057\%. 
In Figure~\ref{fig:testAccuracyHistogram}(b), when the compression ratio is 8, HashedNets provides smaller testing errors than that of SmallNets, 
i.e., with means of 1.43\% v.s. 1.76\%, and standard deviations of 0.052\% v.s. 0.070\%, respectively. 
ThinNets provides slightly better testing errors than that of HashedNets, 
i.e., with means of 1.32\% v.s. 1.44\%, and standard deviations of 0.060\% v.s. 0.056\%. 
In other words, both HashedNets and ThinNets can achieve higher and more robust accuracy with improvements in both mean and standard deviation than SmallNets for this compression ratio. 
With the compression ratio increasing to 64, as shown in Figure~\ref{fig:testAccuracyHistogram}(c), the benefit of HashedNets is more significant. 
The mean of testing errors for HashedNets degrades to 2.80\%, while that for SmallNets increases to 6.09\%. 
The errors for SmallNets are more instable due to larger standard deviations. 
Meanwhile, HashedNets can also provide better accuracy than ThinNets, 
i.e., with means of 2.80\% v.s. 5.03\%, and standard deviations of 0.090\% v.s. 0.196\%. 
When the compression ratio is 128, as shown in Figure~\ref{fig:testAccuracyHistogram}(d), 
HashedNets achieves a mean accuracy of 4.20\% and a standard deviation of 0.116\%, 
which is much better than that of ThinNets, a mean accuracy of 11.09\% and a standard deviation of 0.160\%, 
and that of SmallNets, a mean accuracy of 10.28\% and a standard deviation of 0.810\%. 
In summary, from the aspect of accuracy degradation, when the compression ratio increases from 1 to 128, there is on average 2.83\% degradation in accuracy for HashedNets, 
while the accuracy of SmallNets degrades by 8.88\% and that of ThinNets degrades by 9.82\%. 
ThinNets may achieve comparable error to HashedNet for small compression ratios (e.g., 1 and 8), while for large compression ratio, HashedNet tends to be more stable.

Figure~\ref{fig:testAccuracyCurve} plots the training curves of different networks under different compression ratios  with random seed 100.  
The testing errors align with the observation from Figure~\ref{fig:testAccuracyHistogram}. 
That is, HashedNets provides high and stable accuracy across various compression ratios; 
ThinNets achieves good accuracy for small compression ratios (e.g., 1 and 8, the accuracy is close to that of HashedNets), 
while degrades significantly with the increase of compression ratios; 
SmallNets are also very sensitive to large compression ratios like ThinNets. 

We further verify the full rank assumption of weight matrices in HashedNets. 
Figure~\ref{fig:ConditionNumbers} plots the distributions of minimum and maximum singular values, condition numbers, and stable ranks of the two weight matrices $W_1$ and $W_2$ in HashedNets with 1000 hidden units.
The dimensions of $W_1$ is $1000 \times 784$ and that of $W_2$ is $10 \times 1000$. 
The distributions are extracted from the aforementioned 50 runs. 
Figure~\ref{fig:Singular_Values_for_W1_seed100_CR64} (in the Appendix) gives one example of all singular values sorted from large to small in one experiment. 
All the singular values and condition numbers are distributed in reasonable scales, i.e., neither too close to zero, nor too large.  
This experiment indicates that the assumption of full rank holds in practice. 
Same set of figures are also provided for HashedNets with 500 hidden units, as shown in Figure~\ref{fig:ConditionNumbers_nhu500} (in the Appendix), 
where the dimensions of $W_1$ is $500 \times 784$ and that of $W_2$ is $10 \times 500$. 

\section{Conclusion}
In this paper, we study the theoretical properties of hashing-based neural networks.
We show that (i) parameter-reduced neural nets have uniform approximation power on inputs from any low-dimensional subspace or smooth  and well-conditioned manifold; (ii) one-hidden-layer HashedNets have similar recovery guarantee to that of fully connected neural nets.
We also empirically explore an alternative compression scheme, ThinNets, which is a very interesting direction for further study, 
so we plan to explore its property and theoretical insights in the future.

\section*{Acknowledgments}

This project was done while Zhao was visiting Harvard and hosted by Prof. Jelani Nelson. In 2017, Prof. Jelani Nelson and Prof. Piotr Indyk co-taught a class ``Sketching Algorithms for Big Data''. One of the class project is, ``Understanding hashing trick of neural network.'' That is the initialization of this project. The authors would like to thank Zhixian Lei for inspiring us of this project. 

The authors would like to appreciate Wei Hu for his generous help and contribution to this project. 

The authors would like to thank Rasmus Kyng, Eric Price, Zhengyu Wang and Peilin Zhong for useful discussions.





%



\bibliographystyle{abbrvnat}
\bibliography{ref}

\appendix
\onecolumn
\section*{Appendix}
\section{Notation}

For any positive integer $n$, we use $[n]$ to denote the set $\{1,2,\cdots,n\}$.
For random variable $X$, let $\mathbb{E}[X]$ denote the expectation of $X$ (if this quantity exists).
For any vector $ x\in \mathbb{R}^n$, we use $\|x\|_2$ to denote its $\ell_2$ norm.

We provide several definitions related to matrix $A$.
Let $\det(A)$ denote the determinant of a square matrix $A$. Let $A^\top$ denote the transpose of $A$. Let $A^\dagger$ denote the Moore-Penrose pseudoinverse of $A$. Let $A^{-1}$ denote the inverse of a full rank square matrix. Let $\| A\|_F$ denote the Frobenius norm of matrix $A$. Let $\| A\|$ denote the spectral norm of matrix $A$. Let $\sigma_i(A)$ to denote the $i$-th largest singular value of $A$.

For any function $f$, we define $\widetilde{O}(f)$ to be $f\cdot \log^{O(1)}(f)$. In addition to $O(\cdot)$ notation, for two functions $f,g$, we use the shorthand $f\lesssim g$ (resp. $\gtrsim$) to indicate that $f\leq C g$ (resp. $\geq$) for an absolute constant $C$. We use $f\eqsim g$ to mean $cf\leq g\leq Cf$ for constants $c,C$.

We state a trivial fact that connects $\ell_2$ norm with $\ell_{\infty}$ norm.
\begin{fact}\label{fac:l2_sqrtn_linf}
For any vector $x \in \R^n$, we have $\| x \|_2 \leq \sqrt{n} \| x \|_{\infty}$.
\end{fact}

\section{Neural Subspace Embedding}\label{sec:app_subspace}

\subsection{Preliminaries}\label{sec:nural_subspace_embdding_defs}

\begin{definition}[Johnson Lindenstrauss Transform, \cite{jl84}]
A random matrix $S\in \R^{k\times n}$ forms a Johnson-Lindenstrauss transform with parameters $\epsilon,\delta,f$, or JLT($\epsilon,\delta,f$) for short, if with probability at least $1-\delta$, for any $f$-element subset $V \subset \R^n$, for all $v,v'\in V$ it holds that
	\begin{align*}
	| \langle S v , S v' \rangle - \langle v , v' \rangle | \leq \epsilon \| v \|_2 \| v' \|_2.
	\end{align*}
\end{definition}

The well-known \textsc{Count-Sketch} matrix from the data stream literature \cite{ccf02,tz12} is a sub-space embedding and JL matrix. The definition is provided as follows.
\begin{definition}[\textsc{Count-Sketch}]
Let $S$ denote a $s \times n$ matrix. We choose a random hash function $h : [n] \rightarrow [s]$, and choose a random hash function $\sigma : [n]\rightarrow \{-1,+1\}$. We set
\begin{align*}
	S_{j,i} =
	\begin{cases}
	\sigma(i) & \text{~if~} j = h(i), \\
	0 & \mathrm{~otherwise.}
	\end{cases}
\end{align*}
\end{definition}

\textsc{Count-Sketch} matrix gives the following subspace embedding result,
\begin{theorem}[\cite{cw13,nn13}]\label{thm:count_sketch_subspace}
For any $0 < \delta < 1$, and for $S$ a \textsc{Count-Sketch} matrix $s = O(d^2 /(\delta \epsilon^2))$ rows, then with probability $1-\delta$, for any fixed $n \times d$ matrix $U$, $S$ is a $(1\pm\epsilon)$ $\ell_2$-subspace embedding for $U$. The matrix product $S \cdot U$ can be computed in $O(\nnz(U))$ time. Further, all of this holds if the hash function $h$ defining $S$ is only pairwise independent, and the sign function $\sigma$ defining $S$ is only 4-wise independent.
\end{theorem}

\begin{definition}[Oblivious Subspace Embedding(OSE), Definition 2.2 in \cite{w14}]
Suppose $\Pi$ is a distribution on $s \times n$ matrices $S$, where $s$ is a function of $n,d,\epsilon$ and $\delta$. Suppose that with probability at least $1-\delta$, for any fixed $n \times d$ matrix $U$, a matrix $S$ drawn from distribution $\Pi$ has the property that $S$ is a $(1\pm \epsilon)$ $\ell_2$-subspace embedding for $U$. Then we call $\Pi$ an $(\epsilon,\delta)$ oblivious $\ell_2$-subspace embedding.
\end{definition}

\cite{nn13} provides some other constructions for subspace embedding, 
\begin{definition}[\textsc{Sparse-Embedding}]
Let $S$ denote a $s \times n$ matrix. For each $i,j$, $S_{i,j} \in \{0, 1/\sqrt{t}, -1/\sqrt{t} \}$. For a random draw $S$, let $\delta_{i,j}$ be a indicator random variable for the event $S_{i,j} \neq 0$, and write $S_{i,j} = \delta_{i,j} \sigma_{i,j} / \sqrt{t}$, where the $\sigma_{i,j} $ are random signs. $S$ satisfies the following two properties, for each $j \in [n]$, $\sum_{i=1}^s \delta_{i,j} = t$; for any set $T \subseteq [s] \times [n]$, $\E[\prod_{(i,j) \in T} \delta_{i,j} \leq (t/m)^{|T|}]$.
\end{definition}

\begin{lemma}[Lemma 2.2 in \cite{w14}]\label{lem:net_lemma}
Let $A = \{ \R^n ~|~ y = U x, \forall x \in \R^d, \| y \|_2=1 \} $, for any $0 < \gamma < 1$, there exists a $\gamma$-net ${\cal N}$ of $A$ for which $|{\cal N}| \leq (1+ 4/\gamma)^d$.
\end{lemma}

\subsection{Proof of Theorem~\ref{thm:subspace_one_hidden_I_formal}}

\begin{proof}
Using Theorem~\ref{thm:sparse_embedding_subspace}, we choose $s = d \epsilon'^{-2} \poly \log (n d / (\epsilon' \delta))$, with probability $1-\delta$, we have : for $n_2$ fixed vectors $w_1,\cdots,w_{n_2}$ and for all $x \in \text{colspan}(U)$, \vspace{-2mm}
\begin{align*}
| \langle S w_i, S x \rangle - \langle w_i, x \rangle | \leq \epsilon' \| w_i \|_2 \| x \|_2,  \forall i \in [n_1].
\end{align*}
Using the Lipschitz property of $\phi_i$, we can show that
\begin{align}\label{eq:one_hidden_I_phiSSx}
 | \phi_i(w_i^\top x) - \phi_i( \langle S w_i, S x \rangle ) | 
 \leq & ~ L | w_i^\top x - \langle S w_i, S x \rangle | \notag \\
 \leq & ~ L \epsilon' \| w_i \|_2 \| x \|_2 \notag \\
 \leq & ~ \epsilon' L B A, 
\end{align}
where the first step follows by Property of function $\phi_i$, the last step follows by $\| w_i \|_2 \leq B, \|x \|_2 \leq A$.
It remains to bound {
\begin{align*}
 \left| \sum_{i=1}^{n_2} v_i \phi_i( w_i^\top x ) - v_i \phi_i( w_i^\top S^\top S x ) \right|
\leq & ~ \sum_{i=1}^{n_2} |v_i| \cdot | \phi_i( w_i^\top x ) - \phi_i( w_i^\top S^\top S x ) | \\
\leq & ~ \| v \|_1 \max_{i \in [n_2] } | \phi_i( w_i^\top x ) - \phi_i( w_i^\top S^\top S x ) | \\
\leq & ~ \| v \|_1 \epsilon' L B A  \leq ~ \epsilon' L B^2 A \leq ~ \epsilon,
\end{align*}}
where the first step follows by triangle inequality, the third step follows by Eq.~\eqref{eq:one_hidden_I_phiSSx}, the fourth step follows by $\| v \|_1 \leq B$, and the last step follows from $\epsilon' = \Theta( \epsilon / (B^2 A L) )$. Therefore, it suffices to choose
\begin{align*}
s = d B^4 A^2 L^2 \epsilon^{-2} \poly\log (nBA L/(\epsilon \delta)).
\end{align*}
\end{proof}

\subsection{Proof of Theorem~\ref{thm:subspace_one_hidden_II}}
\begin{proof}
The proof includes three steps.
The first step is similar to the proof of Theorem~\ref{thm:subspace_one_hidden_I_formal}. 
Using Theorem~\ref{thm:sparse_embedding_subspace}, we choose $s_1 = d \epsilon_1^{-2} \poly \log (n d / (\epsilon_1 \delta))$, with probability $1-\delta/2$, we have : for $n_2$ fixed vectors $w_1,\cdots,w_{n_2}$ and for all $x \in \text{colspan}(U)$,
$
| \langle S_1 w_i, S_1 x \rangle - \langle w_i, x \rangle |\leq  \epsilon_1 \| w_i \|_2 \| x \|_2,  \forall i \in [n_2].
$
Next, we consider the column space of $U$, which we call $P_1$, defined as follows
\begin{align*}
P_1 = \{ y_1 \in \R^{n_1} ~|~ y_1 = U x, \| y_1 \|_2 \leq 1 , \forall x \in \R^d \}.
\end{align*}
Let ${\cal N}_1$ denote the $\gamma_1$-net of $P_1$, by Lemma~\ref{lem:net_lemma}, $|{\cal N}_1|= 2^{O(d \log ( 1/\gamma_1 ))}$.
By definition, for each $y_1 \in P_1$, there exists a vector $z \in {\cal N}_1$ such that
$
\| y_1 - z_1 \|_2 \leq \gamma_1.
$

Let $P_2$ and ${\cal N}_2$ be defined as follows,
\begin{align*}
P_2 = & ~ \{ y_2 \in \R^{n_2} ~|~ y_2 = \phi( W^\top S_1^\top S_1 y_1) , \forall y_1 \in P_1 \}. \\
{\cal N}_2 = & ~ \{ z_2 \in \R^{n_2} ~|~ z_2 = \phi( W^\top S_1^\top S_1 z_1 ) , \forall z_1 \in {\cal N}_1 \}.
\end{align*}
Then we want to show the following claim. (The proof can be found in Appendix~\ref{sec:app_one_hidden}) \vspace{-2mm}
\define{cla:subspace_one_hidden_II_recursive}{Claim}{{\rm(Recursive $\epsilon$-net){\bf.}}
Let $\gamma_2 = \gamma_1 / ( \sqrt{n_2} L (1+\epsilon_1) B)$, then
${\cal N}_2$ is an $\gamma_2$-net to $P_2$.
}
\state{cla:subspace_one_hidden_II_recursive}

\vspace{-2mm}
Now, we choose a sketching matrix $S_2 \in \R^{s_2 \times n_2}$ with $s_2 = d \epsilon_2^{-2} \poly(n d / (\epsilon_2 \delta))$, with probability $1-\delta/2$, we have : a vector $v\in \R^{n_2}$ and for all $y_1 \in P_2$,
\begin{align*}
| \langle S_2 v, S_2 y_2 \rangle - \langle v, y_2 \rangle | \leq \epsilon_2 \| v \|_2 \| y_2 \|_2 .
\end{align*}
Using triangle inequality, we can bound the error term,
\allowdisplaybreaks 
\begin{align*}
 | \langle v, \phi(W^\top x) \rangle - \langle S_2 v, S_2 \phi(W^\top S_1^\top S_1 x) \rangle | 
\leq & ~ | \langle v, \phi(W^\top x) \rangle - \langle v, \phi(W^\top S_1^\top S_1 x) \rangle | \\
 + & ~| \langle  v,  \phi(W^\top S_1^\top S_1 x) \rangle - \langle S_2 v, S_2 \phi(W^\top S_1^\top S_1 x) \rangle | \\
 \lesssim & ~ \epsilon_2 LB^2A + \epsilon_1 B^2 A L \leq ~ \epsilon,
 \end{align*}
 where the first step follows from triangle inequality, the second step follows from the Property of $S_2$, the third step follows from Claim~\ref{cla:subspace_one_hidden_II_1} and Claim~\ref{cla:subspace_one_hidden_II_2}, and the last step follows from $\epsilon_1, \epsilon_2 \lesssim \epsilon /( L B^2 A )$.
\end{proof}

We list the two Claims here and delay the proofs into Appendix~\ref{sec:app_one_hidden}.
\define{cla:subspace_one_hidden_II_1}{Claim}{
$| \langle  v,  \phi(W^\top S_1^\top S_1 x) \rangle - \langle S_2 v, S_2 \phi(W^\top S_1^\top S_1 x) \rangle | \lesssim \epsilon_2 L B^2 A.$
}
\state{cla:subspace_one_hidden_II_1}

\define{cla:subspace_one_hidden_II_2}{Claim}{
$
| \langle v, \phi(W^\top x) -\langle v, \phi(W^\top S_1^\top S_1 x) \rangle | \leq L B^2 A \epsilon_1.
$
}
\state{cla:subspace_one_hidden_II_2}

\subsection{One hidden layer}\label{sec:app_one_hidden}
In this Section, we provide the proofs of some Claims used for one hidden layer case.

\restate{cla:subspace_one_hidden_II_recursive}
\begin{proof}
For each point $y_2 \in P_2$, there must exists a point $y_1 \in P_1$ such that
$
y_2 = \phi (W^\top S_1^\top S_1 y_1).
$

Since $y_1 \in P_1$ and ${\cal N}_1$ is the $\epsilon_1$-net of $P_1$. Thus, there must exists a vector $z_1 \in {\cal N}_1$ such that
$
\| y_1 - z_1 \|_2 \leq \gamma_1.
$

According to the definition ${\cal N}_2$, there must exists a point $z_2 \in {\cal N}_2$ such that
$
z_2 = \phi(W_1^\top S_1^\top S_1 z_1).
$
Now, let's consider the $\| y_2-z_2\|_2$,
\allowdisplaybreaks {
\begin{align*}
 \| y_2-z_2\|_2 
\leq & ~ \sqrt{n_2} \| y_2 - z_2 \|_{\infty} \\
= & ~ \sqrt{n_2} \| \phi(W^\top S_1^\top S_1 y_1) - \phi(W^\top S_1^\top S_1 z_1) \|_{\infty} \\
= & ~ \sqrt{n_2} \max_{i \in [n_2]} | \phi_i (w_i^\top S_1^\top S_1 y_1) - \phi_i (w_i^\top S_1^\top S_1 z_1) | \\
\leq & ~ \sqrt{n_2} \max_{i \in [n_2]}  L | w_i^\top S_1^\top S_1 y_1 - w_i^\top S_1^\top S_1 z_1 | \\
= & ~ \sqrt{n_2} \max_{i \in [n_2]} L | \langle S_1 w_i, S_1 (y_1 - z_1) \rangle| \\
\leq & ~ \sqrt{n_2} \max_{i\in [n_2]} L (1+\epsilon_1) \| w_i \|_2 \| y_1 - z_1 \|_2 \\
\leq & ~ \sqrt{n_2} L (1+\epsilon_1) B \gamma_1 \leq  ~ \gamma_2,
\end{align*}}
where the first step follows from $\| \cdot \|_2 \leq \sqrt{\text{dim}} \cdot \| \cdot \|_{\infty}$, the second step follows from definition of $y_1$ and $z_1$, the third step follows from the definition of $\|\cdot\|_{\infty}$, the fourth step follows from the property of the activation function ($L$-Lipshitz), the sixth step follows from that $S_1$ provides a subspace embedding, and the last step follows from the definition of $\gamma_2$.
\end{proof}

\restate{cla:subspace_one_hidden_II_1}
\begin{proof}
\begin{align*}
 | \langle  v,  \phi(W^\top S_1^\top S_1 x) \rangle - \langle S_2 v, S_2 \phi(W^\top S_1^\top S_1 x) \rangle | 
\leq & ~ \epsilon_2 \| v\|_2  \cdot \| \phi(W^\top S_1^\top S_1 x) \|_2 \\
\leq & ~ \epsilon_2 \| v\|_2  \cdot L \max_{i\in [n_2]} | w_i^\top S_1^\top S_1 x | \\
\leq & ~ \epsilon_2 \| v\|_2 \cdot L \max_{i\in [n_2]} (1\pm\epsilon_1) \| w_i \|_2 \| x \|_2 \\
\lesssim & ~ \epsilon_2 L B^2 A,
\end{align*}
where the first step follows from that $S_2$ is a $(1\pm\epsilon_2)$ $\ell_2$ subspace embedding, the second step follows from the property of $\phi_i$ ($L$-lipshitz), the third step follows from that $S_1$ is $(1\pm\epsilon_1)$ $\ell_2$ subspace embedding, the fourth step follows from that $\| v \|_2, \| W \|_2 \leq B$ and $\| x \|_2 \leq A$.
\end{proof}

\restate{cla:subspace_one_hidden_II_2}

\begin{proof}
\begin{align*}
 | \langle v, \phi(W^\top x) -\langle v, \phi(W^\top S_1^\top S_1 x) \rangle | 
\leq & ~ \|v \|_2 \cdot \| \phi(W^\top x) -  \phi(W^\top S_1^\top S_1 x) \|_2 \\
\leq & ~ \| v\|_2 \cdot  L\max_{i\in [n_2]} | w_i^\top x -  w_i^\top S_1^\top S_1 x | \\
\leq & ~ \| v \|_2 L \epsilon_1 \max_{i\in [n_2]} \| w_i \|_2 \| x\|_2 \\
\leq & ~ L B^2 A \epsilon_1,
\end{align*}
where the first step follows from Cauchy-Shawrz inequality, the second step follow by property of activation function, the third step follows from $S_1$ is an $(1\pm\epsilon_1)$ subspace embedding, and last step follows from that $\| v \|_2, \|W \|_2 \leq B$, $\| x \|_2 \leq A$.
\end{proof}

\subsection{Two hidden layers}
The goal of section is to present Theorem~\ref{thm:subspace_two_hidden}.
\begin{theorem}[Neural oblivious subspace embedding two hidden layer neural networks]\label{thm:subspace_two_hidden}
Given parameters $n_3 \geq 1, n_2 \geq 1, n_2\geq 1 ,\epsilon \in (0,1), \delta \in (0,1)$ and $n=\max(n_3,n_2,n_1)$.
 Given $n_2$ activation functions $\phi_{1,i} :\R \rightarrow \R$ with $L$-Lipshitz and normalized by $1/\sqrt{n_2}$, $n_3$ activation functions $\phi_{2,i} : \R \rightarrow \R$ with $L$-Lipshitz and normalized by $1/\sqrt{n_3}$, a fixed matrix $U \in \R^{n_1 \times d}$, two weight matrices $W_1 \in \R^{n_2 \times n_1}$ and $W_2 \in \R^{n_3 \times n_2}$ with (the $i_1$-th column of $W_1$) $w_{1,i_1} \in {\cal B}_2(B,n_1)$ and (the $i_2$-th column of $W_2$) $w_{2,i_2} \in {\cal B}_2(B,n_2)$, a weight vector $v\in \R^{n_3}$ with $v \in {\cal B}_{2}(B,n_3)$. Choose a \textsc{SparseEmbedding} matrix $S_1 \in \R^{s_1 \times n_1}$, $S_2 \in \R^{s_2 \times n_2}$, $S_3 \in \R^{s_3 \times n_3}$ with 
 \begin{align*}
 s_1,s_2, s_3= O( d  L^4 B^6 A^2 \epsilon^{-2} \poly \log (n L B A / (\epsilon \delta) )  ) ,
 \end{align*}
 then with probability $1-\delta$, we have : for all $x \in \mathrm{colspan}(U) \cap {\cal B}_2(A, n_1)$, 
\begin{align*}
| \langle v, \phi_2 (W_2^\top \phi_1(W_1^\top x) ) \rangle - \langle S_3 v, S_3 \phi_2 (W_2^\top S_2^\top S_2  \phi_1 (W_1^\top S_1^\top S_1 x) ) \rangle | \leq  \epsilon  .
\end{align*}
\end{theorem}
\begin{proof}
It follows by combining Section~\ref{sec:two_hidden_recursive_net} and Lemma~\ref{lem:two_hidden_bounding_error}.
\end{proof}

\subsubsection{Recursive $\epsilon$-net argument}\label{sec:two_hidden_recursive_net}

Using Theorem~\ref{thm:sparse_embedding_subspace}, we choose $s_1 = d \epsilon_1^{-2} \poly \log (n d / (\epsilon_1 \delta))$, with probability $1-\delta/2$, we have : for $n_2$ fixed vectors $w_1,\cdots,w_{n_2}$ and for all $x \in \text{colspan}(U)$,
\begin{align*}
| \langle S_1 w_i, S_1 x \rangle - \langle w_i, x \rangle |\leq  \epsilon_1 \| w_i \|_2 \| x \|_2,  \forall i \in [n_2].
\end{align*} 

We define $P_1$ as follows
\begin{align*}
P_1 = \{ y_1 \in \R^{n_1} ~|~ y_1 = U x, \| y_1 \|_2 \leq 1 , \forall x \in \R^d \}. 
\end{align*}
Let ${\cal N}_1$ denote the $\gamma_1$-net of $P_1$, by Lemma~\ref{lem:net_lemma}, $|{\cal N}_1|= 2^{O(d \log ( 1/\gamma_1 ))}$.
By definition For each $y_1 \notin P_1$, there exists a vector $z \in {\cal N}_1$ such that
\begin{align*}
\| y_1 - z_1 \|_2 \leq \gamma_1. 
\end{align*}

Let $P_2$ be defined as follows,
\begin{align*}
P_2 = \{ y_2 \in \R^{n_2} ~|~ y_2 = \phi_1( W_1^\top S_1^\top S_1 y_1) , \forall y_1 \in P_1 \}.  
\end{align*}
We define ${\cal N}_2$ to be 
\begin{align*}
{\cal N}_2 = \{ z_2 \in \R^{n_2} ~|~ z_2 = \phi_1( W_1^\top S_1^\top S_1 z_1 ) , \forall z_1 \in {\cal N}_1 \}. 
\end{align*}
Then we want to show

\begin{claim}
Let $\gamma_2 = \gamma_1 / ( L (1+\epsilon_1) B)$, then
${\cal N}_2$ is an $\gamma_2$-net to $P_2$.
\end{claim}

\begin{proof}
For each point $y_2 \in P_2$, there must exists a point $y_1 \in P_1$ such that
\begin{align*}
y_2 = \phi (W_1^\top S_1^\top S_1 y_1). 
\end{align*}
Since $y_1 \in P_1$ and ${\cal N}_1$ is the $\epsilon_1$-net of $P_1$. Thus, there must exists a vector $z_1 \in {\cal N}_1$ such that
\begin{align*}
\| y_1 - z_1 \|_2 \leq \gamma_1. 
\end{align*}
According to the definition ${\cal N}_2$, there must exists a point $z_2 \in {\cal N}_2$ such that
\begin{align*}
z_2 = \phi(W_1^\top S_1^\top S_1 z_1).
\end{align*}
Now, let's consider the $\| y_2-z_2\|_2$,
\allowdisplaybreaks 
\begin{align*}
 \| y_2-z_2\|_2
\leq & ~ \sqrt{n_2} \| y_2 - z_2 \|_{\infty} \\
= & ~ \sqrt{n_2} \| \phi_1 (W_1^\top S_1^\top S_1 y_1) - \phi_1 (W_1^\top S_1^\top S_1 z_1) \|_{\infty} \\
= & ~ \sqrt{n_2} \max_{i \in [n_2]} | \phi_1 (w_{1,i_1}^\top S_1^\top S_1 y_1) - \phi_1 (w_{1,i_1}^\top S_1^\top S_1 z_1) | \\
= & ~ \max_{i_1 \in [n_2]}  L | w_{1,i_1}^\top S_1^\top S_1 y_1 - w_{1,i_1}^\top S_1^\top S_1 z_1 | \\
= & ~  \max_{i_1 \in [n_2]} L | \langle S_1 w_{1,i_1}, S_1 (y_1 - z_1) \rangle| \\
\leq & ~  \max_{i_1 \in [n_2]} L (1+\epsilon_1) \| w_{1,i_1} \|_2 \| y_1 - z_1 \|_2 \\
\leq & ~  L (1+\epsilon_1) B \gamma_1 \\
\leq & ~ \gamma_2, 
\end{align*}
where the first step follows by $\| \cdot \|_2 \leq \sqrt{\text{dim}} \cdot \| \cdot \|_{\infty}$, the second step follows by definition of $y_1$ and $z_1$, the third step follows by definition, the fourth step follows by property of activation function ($L$-lipshitz and normalization), the sixth step follows by $S_1$ provides a subspace embedding, the last step follows by definition of $\gamma_2$.
\end{proof}

It is obvious that $| {\cal N}_2 | = | {\cal N}_1 |$.

Let $P_3$ be defined as follows,
\begin{align*}
P_3 = \{ y_3 \in \R^{n_3} ~|~ y_3 = \phi_2( W_2^\top S_2^\top S_2 y_2) , \forall y_2 \in P_2 \}. 
\end{align*}
We define ${\cal N}_2$ to be 
\begin{align*}
{\cal N}_3 = \{ z_3 \in \R^{n_3} ~|~ z_3 = \phi_2( W_2^\top S_2^\top S_2 z_2 ) , \forall z_2 \in {\cal N}_2 \}. 
\end{align*}
Then we want to show
\begin{claim}
Let $\gamma_3 = \gamma_2 / ( L (1+\epsilon_2) B )$, then
${\cal N}_3$ is an $\gamma_3$-net to $P_3$.
\end{claim}

\begin{proof}
For each point $y_3 \in P_3$, there must exists a point $y_2 \in P_2$ such that
\begin{align*}
y_3 = \phi (W_2^\top S_2^\top S_2 y_2). 
\end{align*}
Since $y_2 \in P_2$ and ${\cal N}_2$ is the $\epsilon_2$-net of $P_2$. Thus, there must exists a vector $z_2 \in {\cal N}_2$ such that
\begin{align*}
\| y_2 - z_2 \|_2 \leq \gamma_2. 
\end{align*}
According to the definition ${\cal N}_3$, there must exists a point $z_3 \in {\cal N}_3$ such that
\begin{align*}
z_3 = \phi(W_2^\top S_2^\top S_2 z_2).
\end{align*}
Now, let's consider the $\| y_3-z_3\|_2$,
\allowdisplaybreaks 
\begin{align*}
 \| y_3-z_3\|_2 
\leq & ~ \sqrt{n_3} \| y_3 - z_3 \|_{\infty} \\
= & ~ \sqrt{n_3} \| \phi_2 (W_2^\top S_2^\top S_2 y_2) - \phi_2 (W_2^\top S_2^\top S_2 z_2) \|_{\infty} \\
= & ~ \sqrt{n_3} \max_{i \in [n_3]} | \phi_2 (w_{2,i_2}^\top S_2^\top S_2 y_2) - \phi_1 (w_{2,i_2}^\top S_2^\top S_2 z_2) | \\
= & ~ \max_{i_2 \in [n_3]}  L | w_{2,i_2}^\top S_2^\top S_2 y_2 - w_{2,i_2}^\top S_2^\top S_2 z_2 | \\
= & ~  \max_{i_2 \in [n_3]} L | \langle S_2 w_{2,i_2}, S_2 (y_2 - z_2) \rangle| \\
\leq & ~  \max_{i_2 \in [n_3]} L (1+\epsilon_2) \| w_{2,i_2} \|_2 \| y_2 - z_2 \|_2 \\
\leq & ~  L (1+\epsilon_2) B \gamma_2 \\
\leq & ~ \gamma_3, 
\end{align*}
where the first step follows by $\| \cdot \|_2 \leq \sqrt{\text{dim}} \cdot \| \cdot \|_{\infty}$, the second step follows by definition of $y_2$ and $z_2$, the third step follows by definition, the fourth step follows by property of activation function ($L$-lipshitz and normalization), the sixth step follows by $S_2$ provides a subspace embedding, the last step follows by definition of $\gamma_3$.
\end{proof}

\subsubsection{Bounding the error}

\begin{lemma}[Bounding the error]\label{lem:two_hidden_bounding_error}
\begin{align*}
| \langle v, \phi_2 (W_2^\top \phi_1(W_1^\top x) ) \rangle - \langle S_3 v, S_3 \phi_2 (W_2^\top S_2^\top S_2  \phi_1 (W_1^\top S_1^\top S_1 x) ) \rangle | \leq 2(\epsilon_1 + \epsilon_2 + \epsilon_3 ) L^2 B^3 A.
\end{align*}
\end{lemma}
\begin{proof}
\allowdisplaybreaks
\begin{align*}
& ~ | \langle v, \phi_2 (W_2^\top \phi_1 (W_1^\top x) ) \rangle - \langle S_3 v, S_3 \phi_2 (W_2^\top S_2^\top S_2  \phi_1 (W^\top S_1^\top S_1 x) ) \rangle | \\
\leq & ~ | \langle v, \phi_2 (W_2^\top \phi_1 (W_1^\top x) ) \rangle - \langle v, \phi_2 (W_2^\top \phi_1 (W_1^\top S_1^\top S_1 x) ) \rangle | \\
+ & ~ | \langle v, \phi_2 (W_2^\top \phi_1 (W_1^\top S_1^\top S_1 x) ) \rangle - \langle v, \phi_2 (W_2^\top S_2^\top S_2 \phi_1 (W_1^\top S_1^\top S_1 x) ) \rangle | \\
+ & ~ | \langle v, \phi_2 (W_2^\top S_2^\top S_2 \phi_1 (W_1^\top S_1^\top S_1 x) ) - \langle S_3 v, S_3 \phi_2 (W_2^\top S_2^\top S_2  \phi_1(W_1^\top S_1^\top S_1 x) ) \rangle | \\
\leq & ~ 2(\epsilon_1 + \epsilon_2 + \epsilon_3) L^2 B^3 A,
\end{align*}
where the first step follows by triangle inequality, and the last step follows by Claim~\ref{cla:two_hidden_1}, \ref{cla:two_hidden_2}, \ref{cla:two_hidden_3}.
\end{proof}

\begin{claim}\label{cla:two_hidden_1}
$| \langle v, \phi_2 (W_2^\top \phi_1 (W_1^\top x) ) \rangle - \langle v, \phi_2 (W_2^\top \phi_1 (W_1^\top S_1^\top S_1 x) ) \rangle | \leq \epsilon_1 L^2 B^3 A$.
\end{claim}
\begin{proof}
We have,
\begin{align*}
 & ~ | \langle v, \phi_2 (W_2^\top \phi_1 (W_1^\top x) ) \rangle - \langle v, \phi_2 (W_2^\top \phi_1 (W_1^\top S_1^\top S_1 x) ) \rangle | \\
\leq & ~ \| v \|_2 \cdot \| \phi_2 (W_2^\top \phi_1 (W_1^\top x) ) - \phi_2 (W_2^\top \phi_1 (W_1^\top S_1^\top S_1 x) ) \|_{2} \\
\leq & ~ \| v \|_2 \cdot \sqrt{n_3} \cdot \max_{i_2\in [n_3]}  | \phi_2 (w_{2,i_2}^\top \phi_1 (W_1^\top x) ) - \phi_2(w_{2,i_2}^\top \phi_1 (W_1^\top S_1^\top S_1 x) )  | \\
\leq & ~ \| v \|_2 \cdot L \cdot \max_{i_2\in [n_3]} | w_{2,i_2}^\top \phi_1 (W_1^\top x) - w_{2,i_2}^\top \phi_1(W_1^\top S_1^\top S_1 x) | \\
\leq & ~ \| v \|_2 \cdot L \cdot \max_{i_2\in [n_3]} \| w_{2,i_2} \|_2 \| \phi_1(W_1^\top x) - \phi_1(W_1^\top S_1^\top S_1 x) \|_2 \\
\leq & ~ \| v \|_2 \cdot L \cdot \max_{i_2\in [n_3]} \| w_{2,i_2} \|_2 \max_{i_1 \in [n_2]} \sqrt{n_2} L | \phi_1 ( w_{1,i_1}^\top x ) - \phi_1 ( w_{1,i_1}^\top S_1^\top S_1 x ) | \\
\leq & ~ \| v \|_2 \cdot L \cdot \max_{i_2\in [n_3]} \| w_{2,i_2} \|_2 \max_{i_1 \in [n_2]} L |  w_{1,i_1}^\top x  - w_{1,i_1}^\top S_1^\top S_1 x  | \\
\leq & ~ \| v \|_2 \cdot L \cdot \max_{i_2\in [n_3]} \| w_{2,i_2} \|_2 \max_{i_1 \in [n_2]} L \epsilon_1 \| w_{1,i_1} \|_2 \| x \|_2 \\
\leq & ~ \epsilon_1 L^2 B^3 A,
\end{align*}
where the first step follows by Cauchy-Shawrz inequality, the second step follows by Fact~\ref{fac:l2_sqrtn_linf}, the third step follows by property of activation function ($L$-Lipshitz and normalization), the fourth step follows by Cauchy-Shawrz inequality, the fifth step follows by Fact~\ref{fac:l2_sqrtn_linf}, the sixth step follows by property of activation function, the seventh step follows by $S_1$ is a $(1\pm \epsilon_1)$ $\ell_2$ subspace embedding, and the last step follows by $\| v\|_2, \| w_{2,i_2} \|_2, \| w_{1,i_1} \|_2 \leq B$, $\| x \|_2 \leq A$.
\end{proof}

\begin{claim}\label{cla:two_hidden_2}
$| \langle v, \phi_2 (W_2^\top \phi_1 (W_1^\top S_1^\top S_1 x) ) \rangle - \langle v, \phi_2 (W_2^\top S_2^\top S_2 \phi_1 (W_1^\top S_1^\top S_1 x) ) \rangle | \leq 2 \epsilon_2 L^2 B^3 A$.
\end{claim}

\begin{proof}
We have,
\allowdisplaybreaks
\begin{align*}
& ~ | \langle v, \phi_2 (W_2^\top \phi_1 (W_1^\top S_1^\top S_1 x) ) \rangle - \langle v, \phi_2 (W_2^\top S_2^\top S_2 \phi_1 (W_1^\top S_1^\top S_1 x) ) \rangle | \\
\leq & ~ \| v \|_2  \| \phi_2 (W_{2}^\top \phi_1 ( W_1^\top S_1^\top S_1 x ) ) - \phi_2 (W_{2}^\top S_2^\top S_2 \phi_1 ( W_1^\top S_1^\top S_1 x ) )  \|_2 \\
\leq & ~ \| v \|_2 \cdot \max_{i_2 \in [n_3] } \sqrt{n_3} \cdot | \phi_2( w_{2,i_2}^\top \phi_1 ( W_1^\top S_1^\top S_1 x ) ) - \phi_2( w_{2,i_2}^\top S_2^\top S_2 \phi_1 ( W_1^\top S_1^\top S_1 x ) ) | \\
\leq & ~ \| v \|_2 \cdot \max_{i_2 \in [n_3] } L \cdot | w_{2,i_2}^\top \phi_1 ( W_1^\top S_1^\top S_1 x ) - w_{2,i_2}^\top S_2^\top S_2 \phi_1 ( W_1^\top S_1^\top S_1 x ) | \\
\leq & ~ \| v \|_2 \cdot \max_{i_2 \in [n_3] } L \cdot \epsilon_2 \cdot \| w_{2,i_2} \|_2 \| \phi_1 (W_{1}^\top S_1^\top S_1 x) \|_2 \\
\leq & ~ \| v \|_2 \cdot \max_{i_2 \in [n_3] } L \cdot \epsilon_2 \cdot \| w_{2,i_2} \|_2 \max_{i_1\in [n_2]} \sqrt{n_2} | \phi_1 (w_{1,i_1}^\top S_1^\top S_1 x) | \\
\leq & ~ \| v \|_2 \cdot \max_{i_2 \in [n_3] } L \cdot \epsilon_2 \cdot \| w_{2,i_2} \|_2 \max_{i_1\in [n_2]} L | w_{1,i_1}^\top S_1^\top S_1 x | \\
\leq & ~ \| v \|_2 \cdot \max_{i_2 \in [n_3] } L \cdot \epsilon_2 \cdot \| w_{2,i_2} \|_2 \max_{i_1 \in [n_2]} (1+\epsilon_1) L \| w_{1,i_1} \|_2  \| x \|_2 \\
\leq & ~ 2 \epsilon_2 L^2 B^3 A,
\end{align*}
where the first step follows by Cauchy-Shawrz inequality, the second step follows by Fact~\ref{fac:l2_sqrtn_linf}, the third step follows by property of activation function, the fourth step follows by $S_2$ is a $(1\pm \epsilon_2)$ $\ell_2$ subspace embedding, the fifth step follows by Fact~\ref{fac:l2_sqrtn_linf}, the sixth step follows by property of activation function, the seventh step follows by $S_1$ is $a$ $(1\pm \epsilon_1)$ $\ell_2$ subspace embedding, and the last step follows by $\| v \|_2, \| w_{2,i_2} \|_2 , \| w_{1,i_1} \|_2 \leq B$ and $\| x \|_2 \leq A$.
\end{proof}

\begin{claim}\label{cla:two_hidden_3}
$| \langle v, \phi_2 (W_2^\top S_2^\top S_2 \phi_1(W_1^\top S_1^\top S_1 x) ) - \langle S_3 v, S_3 \phi_2(W_2^\top S_2^\top S_2  \phi_1 (W^\top S_1^\top S_1 x) ) \rangle | \leq 2 \epsilon_3 L^2 B^3 A $.
\end{claim}

\begin{proof}
We have
\allowdisplaybreaks
\begin{align*}
& ~ | \langle v, \phi_2 (W_2^\top S_2^\top S_2 \phi_1 (W_1^\top S_1^\top S_1 x) ) - \langle S_3 v, S_3 \phi_2 (W_2^\top S_2^\top S_2  \phi_1 (W^\top S_1^\top S_1 x) ) \rangle | \\
\leq & ~ \epsilon_3 \| v \|_2 \cdot \| \phi_2 (W_2^\top S_2^\top S_2 \phi_1 (W_1^\top S_1^\top S_1 x) )  \|_2 \\
\leq & ~ \epsilon_3 \| v \|_2 \cdot L \max_{i_2 \in [n_3]} | w_{2,i_2}^\top S_2^\top S_2 \phi_1 (W_1^\top S_1^\top S_1 x) | \\
\leq & ~ \epsilon_3 \| v \|_2 \cdot L \max_{i_2 \in [n_3]} (1+\epsilon_2)  \| w_{2,i_2} \|_2 \cdot \| \phi_1 (W_1^\top S_1^\top S_1 x) \|_2 \\
\leq & ~ \epsilon_3 \| v \|_2 \cdot L \max_{i_2 \in [n_3]} (1+\epsilon_2)  \| w_{2,i_2} \|_2 \cdot L \max_{i_1 \in [n_2]} | w_{1,i_1}^\top S_1^\top S_1 x | \\
\leq & ~ \epsilon_3 \| v \|_2 \cdot L \max_{i_2 \in [n_3]} (1+\epsilon_2)  \| w_{2,i_2} \|_2 \cdot L \max_{i_1 \in [n_2]} (1+\epsilon_1) \| w_{1,i_1} \|_2  \| x \|_2 \\
\leq & ~ 2\epsilon_3 L^2 B^3 A,
\end{align*}
where the first step follows by $S_3$ is a $(1\pm \epsilon_3)$ $\ell_2$ subspace embedding, the second step follows by Fact~\ref{fac:l2_sqrtn_linf} and the property of activation function, the third step follows by $S_2$ is a $(1\pm\epsilon_2)$ $\ell_2$ subspace embedding, the fourth step follows by Fact~\ref{fac:l2_sqrtn_linf}, the fifth step follows by $S_1$ is a $(1\pm\epsilon_1)$ $\ell_1$ subspace embedding, and last step follows $\| v\|_2, \| w_{2,i_2} \|_2, \| w_{1,i_1} \|_2 \leq B$ and $\| x \|_2 \leq A$.
\end{proof}

\subsection{Multiple hidden layers}\label{sec:subspace_multiple_hidden}

The goal of section is to present Theorem~\ref{thm:subspace_multiple_hidden_formal}.
\begin{theorem}[Oblivious neural subspace embedding multiple hidden layer neural networks, restate of Theorem~\ref{thm:subspace_multiple_hidden_informal}]\label{thm:subspace_multiple_hidden_formal}
Given parameters $q\geq 1$, $n_j \geq 1,\forall j \in [q+1] ,\epsilon \in (0,1), \delta \in (0,1)$ and $n=\max_{j \in [q+1]} n_j$. For each $j \in [q]$, for each $i \in [n_{j+1}]$
let $\phi_{j,i} :\R \rightarrow \R$ denote an activation function  with $L$-Lipshitz and normalized by $1/\sqrt{n_{j+1}}$. Given a fixed matrix $U \in \R^{n_1 \times d}$, $q$ weight matrices $W_j \in \R^{n_{j+1} \times n_j}, \forall j \in [q]$ with (the $i_j$-th column of $W_j$) $w_{j,i_j} \in {\cal B}_2(B,n_j)$, a weight vector $v\in \R^{n_{j+1}}$ with $v \in {\cal B}_{2}(B,n_{j+1})$. For each $j \in [q+1]$, we choose a \textsc{SparseEmbedding} matrix $S_1 \in \R^{s_j \times n_j}$ with 
 \begin{align*}
 s_j= O( d q^2 L^{2q} B^{2q+2} A^2 \epsilon^{-2} \poly \log (n q L B A / (\epsilon \delta) )  ). 
 \end{align*}
Then with probability $1-\delta$, we have : for all $x \in \mathrm{colspan}(U) \cap {\cal B}_2(A, n_1)$, 
\begin{align*}
| \langle v, \phi_q ( W_q^\top \cdots \phi_2 (W_2^\top \phi_1(W_1^\top x) ) ) \rangle - \langle S_{q+1} v, S_{q+1} \phi_q( W_q^\top S_q^\top S_q \cdots \phi_2 (W_2^\top S_2^\top S_2  \phi_1 (W_1^\top S_1^\top S_1 x) ) ) \rangle | \leq  \epsilon.
\end{align*}
\end{theorem}
To provide a clear proof, we define some notations to simplify the multiple hidden layer neural network,
\begin{definition}
We first define the base
\begin{align*}
f^{(0)}(x) = \widetilde{f}^{(0)}(x) = x.
\end{align*}
Then we define the general case in a recursive way,
\begin{align*}
f^{(j)}(x) = & ~ \phi_j(W_j^\top f^{(j-1)}(x) ), \forall j \in [q] \\
\widetilde{f}^{(j)}(x) = & ~ \phi_j(W_j^\top S_j^\top S_j \widetilde{f}^{(j-1)}(x) ), \forall j \in [q] \\
f^{(j,l)} \circ \wt{f}^{(l-1)}(x) = & ~ \phi_j(W_j^\top \cdots \phi_{l} ( W_{l}^\top \wt{f}^{(l-1)}(x) ) ), \forall j \in [q], l\in \{1, \cdots, j +1\}.
\end{align*}
Note that $f^{(j,1)} \circ \wt{f}^{(0)}(x) = f^{(j)}(x)$ and $f^{(j,j+1)} \circ \wt{f}^{(j)}(x) = \wt{f}^{(j)}(x)$.
\end{definition}
Using the above definition, it is easy to see that
\begin{align*}
\langle v , \phi_q (W_q^\top \cdots \phi_2 (W_2^\top \phi_1 ( W_1^\top x )  ) ) \rangle = \langle v, f^{(q)}(x) \rangle,
\end{align*}
and
\begin{align*}
 \langle S_{q+1} v, S_{q+1} \phi_q( W_q^\top S_q^\top S_q \cdots \phi_2 (W_2^\top S_2^\top S_2  \phi_1 (W_1^\top S_1^\top S_1 x) ) ) \rangle = \langle S_{q+1} v, S_{q+1} \wt{f}^{(q)}(x) \rangle.
\end{align*}
Thus, we have
\begin{align*}
& ~ |  \langle v, f^{(q)}(x) \rangle - \langle S_{q+1} v, S_{q+1} \wt{f}^{(q)}(x) \rangle | \\
\leq & ~ \sum_{l=1}^q | \langle v, f^{(q,l)} \circ \wt{f}^{(l-1)} (x) \rangle - \langle v, f^{(q,l+1)} \circ \wt{f}^{(l)} (x) \rangle | + | \langle v, \wt{f}^{(q)}(x) - \langle S_{q+1} v, S_{q+1} \wt{f}^{(q)}(x)|. 
\end{align*}
Now the question is how to bound the above $q+1$ terms.

\begin{claim}
$ | \langle v, \wt{f}^{(q)}(x) - \langle S_{q+1} v, S_{q+1} \wt{f}^{(q)}(x)| \leq 4 \epsilon_q L^q B^{q+1} A^q$.
\end{claim}
\begin{proof}
We have,
\begin{align*}
 | \langle v, \wt{f}^{(q)}(x) - \langle S_{q+1} v, S_{q+1} \wt{f}^{(q)}(x)| 
\leq & ~ \epsilon_{q+1} \| v \|_2 \cdot \| \wt{f}^{(q)}(x) \|_2 \\
\leq & ~ \epsilon_{q+1} L^q B^{q+1} A^q \prod_{j=1}^q (1+\epsilon_j) \\
\leq & ~ 4 \epsilon_{q+1} L^q B^{q+1} A^q,
\end{align*}
where the first step follows by $S_{q+1}$ is a $(1\pm \epsilon_{q+1})$ $\ell_2$ subspace embedding, the second step follows by applying the argument in the proof of Lemma~\ref{lem:two_hidden_bounding_error} recursively, and the last step follows by $\epsilon_j < 1/q$.
\end{proof}

Similarly, we have
\begin{claim}
$\forall l \in [q], | \langle v, f^{(q,l)} \circ \wt{f}^{(l-1)} (x) \rangle - \langle v, f^{(q,l+1)} \circ \wt{f}^{(l)} (x) \rangle | \leq 4 \epsilon_l L^{q} B^{q+1} A^q$.
\end{claim}
\begin{proof}
We have 
\begin{align*}
 | \langle v, f^{(q,l)} \circ \wt{f}^{(l-1)} (x) \rangle - \langle v, f^{(q,l+1)} \circ \wt{f}^{(l)} (x) \rangle | 
\leq & ~ \epsilon_l L^{q} B^{q+1} A^q \prod_{i=l-1}^{1} (1+\epsilon_i) \\
\leq & ~ 4 \epsilon_l L^{q} B^{q+1} A^q,
\end{align*}
where the first step follows by applying the argument in the proof of Lemma~\ref{lem:two_hidden_bounding_error} recursively, and the last step follows by $\epsilon_i < 1/p$.
\end{proof}
Putting it all together, 
\begin{align*}
|  \langle v, f^{(q)}(x) \rangle - \langle S_{q+1} v, S_{q+1} \wt{f}^{(q)}(x) \rangle | \leq 4 \left( \sum_{j=1}^{q+1} \epsilon_j \right) L^q B^{q+1} A^q.
\end{align*}
Therefore, we can show the error is small.

We also need the $\epsilon$-net argument for all $j \in [q+1]$, it also follows by applying the proof of Section~\ref{sec:two_hidden_recursive_net} recursively. 

Therefore, we complete the proof.

\section{Recovery Guarantee}\label{app:recovery}

\begin{figure}[tb]
\centering 
{\includegraphics[width=0.37\textwidth]{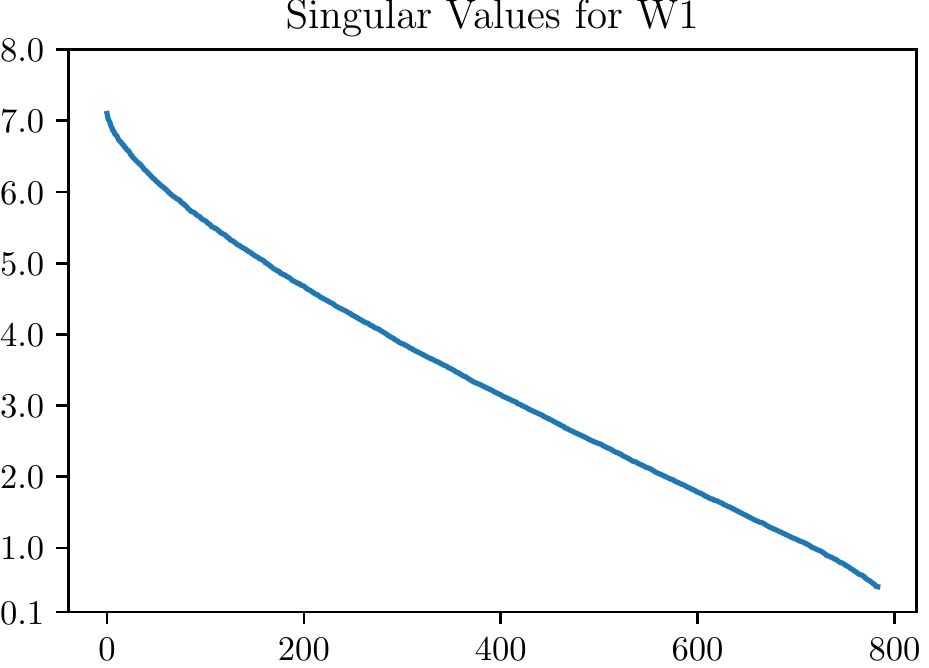}}\label{fig:Singular_Values_for_W1_seed100_CR64:nhu1000}
    \hspace{.3in}
{\includegraphics[width=0.37\textwidth]{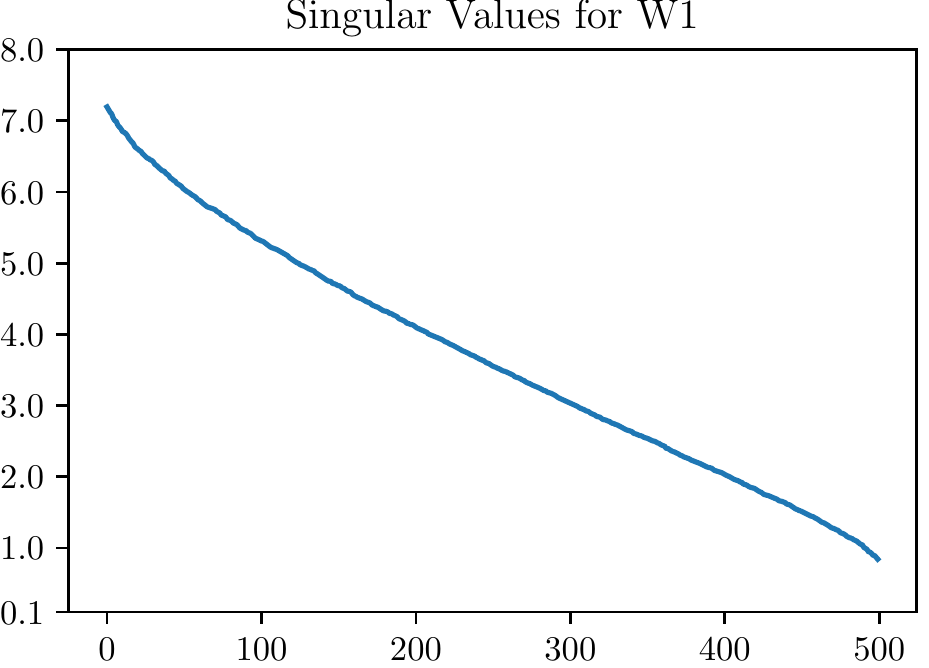}}\label{fig:Singular_Values_for_W1_seed100_CR64:nhu500}
\caption{\small Input dimension is 784. Example of singular values of weight matrix $W_1$ in HashedNets with (a) (left) 1000 hidden units and (b) (right) 500 hidden units from large to small for one random seed. }
\label{fig:Singular_Values_for_W1_seed100_CR64}
\end{figure}

\subsection{Preliminaries}

\begin{theorem}[\cite{br94}]\label{thm:br94_theorem}
Let $k$ be an even integer, and let $X$ be the sum of $n$ $k$-wise independent random variables taking values in $[0,1]$. Let $\mu \E[X]$ and $a >0$. Then we have
\begin{align*}
\Pr[ | X - \mu | > a ] < & ~ 1.1 \left( \frac{nk}{a^2} \right)^{k/2}, \\
\Pr[ | X - \mu | > a ] < & ~ 8 \left( \frac{ k \mu + k^2 }{a^2} \right)^{k/2}.
\end{align*}
\end{theorem}

\subsection{Definitions}
\begin{definition}\label{def:viewing_weights_in_two_ways}
Let $\wh{w} \in \R^{n \times k}$ denote the weights. Let $\wh{w}_i \in \R^{n}$ denote the $i$-th column of $\wh{w}$. Let $\wh{w}_{i,j}$ denote the $j$-th entry of $\wh{w}_i$. Let $h$ denote a hash function that maps $[k] \times [n]$ to $[B]$. Let $w\in \R^{B}$ denote a vector. Then for each $i\in [k], j \in [n]$, we have $w_{ h(i,j) } = \wh{w}_{i,j}$.
\end{definition}

\begin{property}[Property 3.2 in \cite{zsjbd17}]\label{pro:rho}
Let $\alpha_q(\sigma) = \E_{z\sim {\cal N}(0,1) } [ \phi'(\sigma \cdot z) z^q ], \forall q \in \{0,1,2\} $, and $\beta_q(\sigma) = \E_{ z \sim {\cal N}(0,1) } [ \phi'^2(\sigma \cdot z) z^q ], \forall q \in \{0,2\} $. Let $\rho(\sigma)$ denote
\begin{align*}
\min \{ \beta_0(\sigma) - \alpha_0^2 (\sigma) - \alpha_1^2(\sigma) , \beta_2(\sigma) - \alpha_1^2(\sigma) - \alpha_2^2(\sigma), \alpha_0(\sigma) \cdot \alpha_2(\sigma) - \alpha_1^2 (\sigma) \}.
\end{align*}
The first derivative $\phi'(z)$ satisfies that, for all $\sigma > 0$, we have $\rho(\sigma) >0$.
\end{property}

\subsection{Local strong convexity}
Recall the definition
\begin{align}
F_{\cal D}(w) = \frac{1}{2} \E_{(x,y) \sim {\cal D}} \left[ \left( \sum_{i=1}^k v_i^* \cdot \phi \left( \sum_{j=1}^n w_{h(i,j)} \cdot x_j \right) - y \right)^2 \right], \label{def:app_hash_obj_F_D}
\end{align}
\begin{align}
F_{S}(w) = \frac{1}{2} \frac{1}{|S|}  \sum_{(x,y) \in S} \left[ \left( \sum_{i=1}^k v_i^* \cdot \phi \left( \sum_{j=1}^n w_{h(i,j)} \cdot x_j \right) - y \right)^2 \right], \label{def:app_hash_obj_F_S}
\end{align}
where $\phi$ is activation function. All of Section~\ref{app:recovery} assumes that $\phi$ satisfies the Property 3.1, 3.2 and 3.3 in \cite{zsjbd17}.

We use ${\bf 1}_{h(i,j) = p}$ to denote the indicator function which means
\begin{align*}
{\bf 1}_{h(i,j)} =
\begin{cases}
1, & \text{~if~} h(i,j) = p ; \\
0, & \text{~otherwise.}
\end{cases}
\end{align*}
For each $p\in [B]$, we have
\begin{align*}
\frac{\partial F }{ \partial w_p } = \E_{(x,y) \sim {\cal D}} \left[ \left( \sum_{i=1}^k v_i^* \cdot \phi \left( \sum_{j=1}^n w_{h(i,j)}  x_j \right) -y  \right)
\cdot \left( \sum_{i=1}^k v_i^* \cdot \phi' \left( \sum_{j=1}^n w_{h(i,j)} x_j \right) \cdot \left( \sum_{j=1}^n {\bf 1}_{h(i,j) = p} \cdot x_j \right)   \right)
 \right]. 
\end{align*}

For $q \neq p \in [B]$, have
\allowdisplaybreaks
\begin{align*}
\frac{\partial^2 F }{ \partial w_p \partial w_q } = & ~ \E_{(x,y) \sim {\cal D}}  \left[
\left( \sum_{i=1}^k v_i^* \cdot \phi' \left( \sum_{j=1}^n w_{h(i,j)} x_j \right) \cdot \left( \sum_{j=1}^n {\bf 1}_{h(i,j) = p} \cdot x_j \right) \right) \right. \\
& ~ \cdot
\left( \sum_{i=1}^k v_i^* \cdot \phi' \left( \sum_{j=1}^n w_{h(i,j)} x_j \right) \cdot \left( \sum_{j=1}^n {\bf 1}_{h(i,j) = q} \cdot x_j \right) \right) \\
& ~ + \left( \sum_{i=1}^k v_i^* \cdot \phi \left( \sum_{j=1}^n w_{h(i,j)}  x_j \right) -y  \right) \\
& ~ \cdot \left. \left( \sum_{i=1}^k v_i^* \cdot \phi'' \left( \sum_{j=1}^n w_{h(i,j)} x_j \right) \cdot \left( \sum_{j=1}^n {\bf 1}_{h(i,j) = p} \cdot x_j \right) \cdot \left( \sum_{j=1}^n {\bf 1}_{h(i,j) = q} \cdot x_j \right)   \right) \right] \\
= & ~ \E_{(x,y) \sim {\cal D}}  \left[
\left( \sum_{i=1}^k v_i^* \cdot \phi' \left( \sum_{j=1}^n w_{h(i,j)} x_j \right) \cdot \left( \sum_{j=1}^n {\bf 1}_{h(i,j) = p} \cdot x_j \right) \right) \right. \\
& ~ \cdot \left.
\left( \sum_{i=1}^k v_i^* \cdot \phi' \left( \sum_{j=1}^n w_{h(i,j)} x_j \right) \cdot \left( \sum_{j=1}^n {\bf 1}_{h(i,j) = q} \cdot x_j \right) \right) \right], 
\end{align*}
where the last step follows by $\phi'' = 0$.

For $p = q \in [B]$, we have
\allowdisplaybreaks
\begin{align*}
\frac{\partial^2 F }{ \partial w_p^2 } = & ~ \E_{(x,y) \sim {\cal D}}  \left[
\left( \sum_{i=1}^k v_i^* \cdot \phi' \left( \sum_{j=1}^n w_{h(i,j)} x_j \right) \cdot \left( \sum_{j=1}^n {\bf 1}_{h(i,j) = p} \cdot x_j \right) \right) \right. \\
& ~ \cdot
\left( \sum_{i=1}^k v_i^* \cdot \phi' \left( \sum_{j=1}^n w_{h(i,j)} x_j \right) \cdot \left( \sum_{j=1}^n {\bf 1}_{h(i,j) = p} \cdot x_j \right) \right) \\
& ~ + \left( \sum_{i=1}^k v_i^* \cdot \phi \left( \sum_{j=1}^n w_{h(i,j)}  x_j \right) -y  \right) \\
& ~ \cdot \left. \left( \sum_{i=1}^k v_i^* \cdot \phi'' \left( \sum_{j=1}^n w_{h(i,j)} x_j \right) \cdot \left( \sum_{j=1}^n {\bf 1}_{h(i,j) = p} \cdot x_j \right) \cdot \left( \sum_{j=1}^n {\bf 1}_{h(i,j) = p} \cdot x_j \right)   \right) \right] \\
= & ~ \E_{(x,y) \sim {\cal D}}  \left[
\left( \sum_{i=1}^k v_i^* \cdot \phi' \left( \sum_{j=1}^n w_{h(i,j)} x_j \right) \cdot \left( \sum_{j=1}^n {\bf 1}_{h(i,j) = p} \cdot x_j \right) \right) \right. \\
& ~ \cdot \left.
\left( \sum_{i=1}^k v_i^* \cdot \phi' \left( \sum_{j=1}^n w_{h(i,j)} x_j \right) \cdot \left( \sum_{j=1}^n {\bf 1}_{h(i,j) = p} \cdot x_j \right) \right) \right], \\
\end{align*}
where the last step follows by $\phi'' = 0$.

Before presenting the main theorem, we provide some definition here,
\begin{definition}[]
Given $w^* \in \R^B$ and hash function $h : [k] \times [n] \rightarrow [B]$. The matrix $\wh{w}^* \in \R^{k \times n}$ is defined in Definition~\ref{def:viewing_weights_in_two_ways}
For each $i\in [k]$, we use $\sigma_i$ to denote the $i$-th singular value of matrix $\wh{w}^*$. Let $\kappa = \sigma_1 /\sigma_k$, $\lambda = (\prod_{i=1}^k \sigma_i)/\sigma_k^k$. We define $v_{\max} = \max_{i\in [k]} |v_i^*|$ and $v_{\min} = \min_{i\in [k]} |v_i^*|$. For simplicity, let $\nu = v_{\max} / v_{\min}$. Let function $\rho(\sigma)$ be defined in Property~\ref{pro:rho}. Let $p$ be defined in Property 3.1 in \cite{zsjbd17}.
\end{definition}

We present our main theorem here,
\begin{theorem}
Let $F_{\cal D}$ be defined in Eq.~\eqref{def:app_hash_obj_F_D}. Let $N= nk$. Let ${\cal H}$ denote a family of $\Theta(\log N)$-wise independent hash functions that maps $[N] \rightarrow [B]$. Then choosing $h \sim {\cal H}$, we have
\begin{align*}
 \frac{kn}{2B} \cdot A_{\min} \cdot I_B \preceq  \nabla^2 F_{\cal D}(w^*) \preceq \frac{2kn}{B} \cdot A_{\max} \cdot  I_B 
\end{align*}
holds with probability at least $1-1/\poly(N)$, where 
\begin{align*}
A_{\min} = \Omega( \frac{ v_{\min}^2 \rho(\sigma_k) }{\kappa^2 \lambda} ), \text{~and~}  A_{\max} = O(k v_{\max}^2 \sigma_1^{2p}). 
\end{align*}
\end{theorem}
\begin{proof}
It follows by combining Lemma~\ref{lem:hashnet_strong_convex_lower_bound} and Lemma~\ref{lem:hashnet_strong_convex_upper_bound} directly.
\end{proof}

\begin{remark}
The $A_{\min}$ and $A_{\max}$ are the exact same as \cite{zsjbd17}. Once we have reduced the Hessian of HashNet to the Hessian of Fully-connected Net, we run the analysis of \cite{zsjbd17} as a black-box.
\end{remark}

Further we have
\begin{theorem}[]\label{thm:hashnet_strong_convex_S_formal}
Choose a random hash function $h$, for each vector $w$, let matrix $\wh{w}$ be defined according to $w$ and $h$(Definition~\ref{def:viewing_weights_in_two_ways}).
For any $w \in \R^B$ with $\| \wh{w} - \wh{w}^* \| \leq \poly(1/k,1/\lambda,1/\nu, \rho/\sigma_1^{2p}) \cdot \| \wh{w}^* \| $, let $S$ denote a set of i.i.d. samples from distribution ${\cal D}$ and the activation satisfy Property 3.1,3.2,3.3 in \cite{zsjbd17}. Let $F_S$ be defined as Eq.~\eqref{def:app_hash_obj_F_S}. Then for any $t\geq 1$, if
\begin{align*}
| S | \geq n \cdot \poly(\log n, t, k, \nu, \tau, \lambda, \sigma_{1}^{2p}/p),
\end{align*}
we have
\begin{align*}
\frac{nk}{2B} A_{\min} I \preceq \nabla^2 F_S(w^*) \preceq \frac{2nk}{B} A_{\max} I 
\end{align*}
holds with probability at least $1-n^{-\Omega(t)}$.
\end{theorem}

\subsection{Linear convergence of gradient descent}

Finally, following the similar ideas in \cite{zsjbd17}, we can show
\begin{theorem}[Detailed version of Theorem~\ref{thm:recovery_guarantee}]
Let $w^c$ be the current iterate satisfying
\begin{align*}
\| w^c - w^* \| \lesssim \poly(1/k,1/\kappa,1/\lambda,1/\nu,\rho/\sigma_1^{2p}) \cdot \| w^* \|.
\end{align*}
Let $S$ denote a set of i.i.d. samples from distribution ${\cal D}$. Let the activation function satisfy Property 3.1, 3.2 and 3.3(a) in \cite{zsjbd17}. Let $F_S$ be defined as Eq.~\eqref{def:app_hash_obj_F_S}. Let $m_0 = \Theta( v_{\min}^2 \rho(\sigma_k) / (\kappa^2 \lambda) )$ and $M_0 = \Theta( k v_{\max}^2 \sigma_1^{2p})$. For any $t\geq 1$, if we choose
\begin{align*}
|S| \geq n \cdot \poly (\log n, t, \kappa,\lambda, \tau , \nu , \sigma_1^{2p}/\rho)
\end{align*}
and perform gradient descent with step size $1/M_0$ on $F_S(w^c)$ and obtain the next iterate,
$
\wt{w} = w^c - \frac{1}{M_0} \nabla F_S(w^c),
$
then 
$
\| \wt{w} - w^* \|_2^2 \leq \left( 1 - m_0 / M_0 \right) \| w^c - w^* \|_2^2
$
holds with probability at least $1-n^{-\Omega(t)}$.
\end{theorem}

\subsection{Lower bound on eigenvalues}
The goal of this section is to prove Lemma~\ref{lem:hashnet_strong_convex_lower_bound}.
\begin{lemma}[Lower bound]\label{lem:hashnet_strong_convex_lower_bound}
Let $F_{\cal D}$ be defined in Eq.~\eqref{def:app_hash_obj_F_D}. Let $N= nk$. Let ${\cal H}$ denote a family of $\Theta(\log N)$-wise independent hash functions that maps $[N] \rightarrow [B]$. Then choosing $h \sim {\cal H}$, we have
\begin{align*}
\nabla^2 F_{\cal D}(w^*) \succeq \frac{1}{2} \frac{kn}{B} \frac{ v_{\min}^2 \rho(\sigma_k) }{\kappa^2 \lambda}, 
\end{align*}
holds with probability at least $1-1/\poly(N)$.
\end{lemma}

\begin{proof}
Let $a\in \R^{B}$.
The smallest eigenvalue of the Hessian can be calculated by
\begin{align*}
\nabla^2 F(w^*) \succeq & ~ \min_{\| a\| = 1} a^\top \nabla^2 F(w^*) a \cdot I_{B} \\
= & ~ \min_{\| a \| = 1} \E_{x} \left[ \left( \sum_{p\in [B]} a_p \sum_{i=1}^k v_i^* \cdot \phi'\left( \sum_{j=1}^n w_{h(i,j)}^* x_j \right) \cdot \sum_{j=1}^n {\bf 1}_{h(i,j) = p} x_j \right)^2 \right] \cdot I_B, 
\end{align*}
where $I_B$ is a $B \times B$ identity matrix.

The goal is to lower bound the above quantity, we can simplify it in the following way,
\begin{align}\label{eq:expectation_of_hessian_min}
\min_{\| a \| = 1, a \in \R^{B}} \E_{x} \left[ \left( \sum_{p\in [B]} a_p \sum_{i=1}^k v_i^* \cdot \phi'\left( \sum_{j=1}^n w_{h(i,j)}^* x_j \right) \cdot \sum_{j=1}^n {\bf 1}_{h(i,j) = p} x_j \right)^2 \right]. 
\end{align}

For each $i \in [k], j \in [n]$, we define $b_{i,j}$ as follows
\begin{align*}
b_{i,j} = \sum_{p \in [B]} a_p {\bf 1}_{ h_{(i,j) = p } }. 
\end{align*}

Then we can rewrite Eq.~\eqref{eq:expectation_of_hessian_min} by using the definition $b$ in the following way,
\begin{align*}
& ~ \min_{\| a \| = 1, a \in \R^{B} } \E_{x} \left[ \left( \sum_{p\in [B]} a_p \sum_{i=1}^k v_i^* \cdot \phi'\left( \sum_{j=1}^n w_{h(i,j)}^* x_j \right) \cdot \sum_{j=1}^n {\bf 1}_{h(i,j) = p} x_j \right)^2 \right] \\
= & ~ \min_{\| a \| = 1, a \in \R^{B} } \E_{x} \left[ \left( \sum_{i=1}^k \sum_{j=1}^n \sum_{p\in[B]} a_p \cdot {\bf 1}_{h(i,j)=p} \cdot x_j \cdot v_i^* \cdot \phi'( \wh{w}_i^{*\top} \cdot x ) \right)^2 \right] \\
= & ~ \min_{\| a \| = 1, a \in \R^{B} } \E_{x} \left[ \left( \sum_{i=1}^k \sum_{j=1}^n b_{i,j} \cdot x_j \cdot v_i^* \cdot \phi'( \wh{w}_i^{*\top} \cdot x )  \right)^2 \right] \\
= & ~ \min_{\| a \| = 1, a \in \R^{B} } \E_{x} \left[ \left( \sum_{i=1}^k b_i^\top x \cdot v_i^* \cdot \phi'( \wh{w}_i^{*\top} \cdot x )  \right)^2 \right] \\
= & ~ \min_{a \in \R^{B} } \E_{x} \left[ \left( \sum_{i=1}^k b_i^\top x \cdot v_i^* \cdot \phi'( \wh{w}_i^{*\top} \cdot x )  \right)^2 \right]  / \| a \|_2^2 \\
\geq & ~ \min_{a \in \R^{B} } \E_{x} \left[ \left( \sum_{i=1}^k b_i^\top x \cdot v_i^* \cdot \phi'( \wh{w}_i^{*\top} \cdot x )  \right)^2 \right]  / ( 2 \| b \|_2^2 B / (kn) ) \\
= & ~ \min_{b \in \R^{nk} } \E_{x} \left[ \left( \sum_{i=1}^k b_i^\top x \cdot v_i^* \cdot \phi'( \wh{w}_i^{*\top} \cdot x )  \right)^2 \right]  / ( 2 \| b \|_2^2 B / (kn) ) \\
= & ~ \min_{\| b \|_2^2 =1, b \in \R^{nk} } \E_{x} \left[ \left( \sum_{i=1}^k b_i^\top x \cdot v_i^* \cdot \phi'( \wh{w}_i^{*\top} \cdot x )  \right)^2 \right]  / ( 2  B / (kn) ) \\
= & ~ \frac{kn}{2B} \cdot \underbrace{ \min_{\| b \|_2^2 =1, b \in \R^{nk} } \E_{x} \left[ \left( \sum_{i=1}^k b_i^\top x \cdot v_i^* \cdot \phi'( \wh{w}_i^{*\top} \cdot x )  \right)^2 \right] }_{A_{\min}}, 
\end{align*}
where the first step follows by swapping the sums, definition of ${\bf 1}_{h(i,j)=p}$ and the ground-truth weight matrix can be viewed in two ways (see Fact~\ref{def:viewing_weights_in_two_ways}), the second step follows by $\sum_{p\in [B]} a_p \cdot {\bf 1}_{ h(i,j) = p }$, the third step follows by $\sum_{j=1}^n b_{i,j} x_j = b_i^\top x$.

Now the term $A_{\min}$ is exact the same as the Equation (9) in page 34 of arXiv version of \cite{zsjbd17}.
Therefore, we can use Lemma D.6 in page 33 of \cite{zsjbd17}, it provides 
$
A_{\min} \geq v_{\min}^2 \rho(\sigma_k) / ( \kappa^2 \lambda ) .
$
\end{proof}

\subsection{Upper bound on eigenvalues}
The goal of this section is to prove Lemma~\ref{lem:hashnet_strong_convex_upper_bound}.
\begin{lemma}[Upper bound]\label{lem:hashnet_strong_convex_upper_bound}
Let $F$ be defined in Eq.~\eqref{def:app_hash_obj_F_D}. Let $N= nk$. Let ${\cal H}$ denote a family of $\Theta(\log N)$-wise independent hash functions that maps $[N] \rightarrow [B]$. Then choosing $h \sim {\cal H}$, we have
\begin{align*}
\nabla^2 F(w^*) \preceq 2 \frac{kn}{B} k v_{\max}^2 \sigma_1^{2p} I_B, 
\end{align*}
holds with probability at least $1-1/\poly(N)$.
\end{lemma}
\begin{proof}
Let $a\in \R^{B}$.
The largest eigenvalue of the Hessian can be calculated by
\begin{align*}
\nabla^2 F(w^*) \succeq & ~ \max_{\| a\| = 1} a^\top \nabla^2 F(w^*) a \cdot I_{B} \\
= & ~ \max_{\| a \| = 1} \E_{x} \left[ \left( \sum_{p\in [B]} a_p \sum_{i=1}^k v_i^* \cdot \phi'\left( \sum_{j=1}^n w_{h(i,j)}^* x_j \right) \cdot \sum_{j=1}^n {\bf 1}_{h(i,j) = p} x_j \right)^2 \right] \cdot I_B, 
\end{align*}
where $I_B$ is a $B \times B$ identity matrix.

The goal is to lower bound the above quantity, we can simplify it in the following way,
\begin{align}\label{eq:expectation_of_hessian_min_2}
\max_{\| a \| = 1, a \in \R^{B}} \E_{x} \left[ \left( \sum_{p\in [B]} a_p \sum_{i=1}^k v_i^* \cdot \phi'\left( \sum_{j=1}^n w_{h(i,j)}^* x_j \right) \cdot \sum_{j=1}^n {\bf 1}_{h(i,j) = p} x_j \right)^2 \right].  
\end{align}

For each $i \in [k], j \in [n]$, we define $b_{i,j}$ as follows
\begin{align*}
b_{i,j} = \sum_{p \in [B]} a_p {\bf 1}_{ h_{(i,j) = p } }.  
\end{align*}

Then we can rewrite Eq.~\eqref{eq:expectation_of_hessian_min_2} by using the definition $b$ in the following way,
\begin{align*}
& ~ \max_{\| a \| = 1, a \in \R^{B} } \E_{x} \left[ \left( \sum_{p\in [B]} a_p \sum_{i=1}^k v_i^* \cdot \phi'\left( \sum_{j=1}^n w_{h(i,j)}^* x_j \right) \cdot \sum_{j=1}^n {\bf 1}_{h(i,j) = p} x_j \right)^2 \right] \\
= & ~ \max_{\| a \| = 1, a \in \R^{B} } \E_{x} \left[ \left( \sum_{i=1}^k \sum_{j=1}^n \sum_{p\in[B]} a_p \cdot {\bf 1}_{h(i,j)=p} \cdot x_j \cdot v_i^* \cdot \phi'( \wh{w}_i^{*\top} \cdot x ) \right)^2 \right] \\
= & ~ \max_{\| a \| = 1, a \in \R^{B} } \E_{x} \left[ \left( \sum_{i=1}^k \sum_{j=1}^n b_{i,j} \cdot x_j \cdot v_i^* \cdot \phi'( \wh{w}_i^{*\top} \cdot x )  \right)^2 \right] \\
= & ~ \max_{\| a \| = 1, a \in \R^{B} } \E_{x} \left[ \left( \sum_{i=1}^k b_i^\top x \cdot v_i^* \cdot \phi'( \wh{w}_i^{*\top} \cdot x )  \right)^2 \right] \\
= & ~ \max_{a \in \R^{B} } \E_{x} \left[ \left( \sum_{i=1}^k b_i^\top x \cdot v_i^* \cdot \phi'( \wh{w}_i^{*\top} \cdot x )  \right)^2 \right]  / \| a \|_2^2 \\
\geq & ~ \max_{a \in \R^{B} } \E_{x} \left[ \left( \sum_{i=1}^k b_i^\top x \cdot v_i^* \cdot \phi'( \wh{w}_i^{*\top} \cdot x )  \right)^2 \right]  / ( 2 \| b \|_2^2 B / (kn) ) \\
= & ~ \max_{b \in \R^{nk} } \E_{x} \left[ \left( \sum_{i=1}^k b_i^\top x \cdot v_i^* \cdot \phi'( \wh{w}_i^{*\top} \cdot x )  \right)^2 \right]  / ( 2 \| b \|_2^2 B / (kn) ) \\
= & ~ \max_{\| b \|_2^2 =1, b \in \R^{nk} } \E_{x} \left[ \left( \sum_{i=1}^k b_i^\top x \cdot v_i^* \cdot \phi'( \wh{w}_i^{*\top} \cdot x )  \right)^2 \right]  / ( 2  B / (kn) ) \\
= & ~ \frac{kn}{2B} \cdot \underbrace{ \max_{\| b \|_2^2 =1, b \in \R^{nk} } \E_{x} \left[ \left( \sum_{i=1}^k b_i^\top x \cdot v_i^* \cdot \phi'( \wh{w}_i^{*\top} \cdot x )  \right)^2 \right] }_{A_{\max}}, 
\end{align*}
where the first step follows by swapping the sums, definition of ${\bf 1}_{h(i,j)=p}$ and the ground-truth weight matrix can be viewed in two ways (see Fact~\ref{def:viewing_weights_in_two_ways}), the second step follows by $\sum_{p\in [B]} a_p \cdot {\bf 1}_{ h(i,j) = p }$, the third step follows by $\sum_{j=1}^n b_{i,j} x_j = b_i^\top x$.

Now the term $A_{\max}$ is exact the same as the Equation in page 38 of arXiv version of \cite{zsjbd17}.
 Therefore, we can use Lemma D.7 in page 37 of \cite{zsjbd17}, which provides 
$
A_{\max} \leq O(k v_{\max}^2 \sigma_{1}^{2p}) .
$
\end{proof}

\subsection{Analysis of Hashing schemes}
\begin{claim}
With probability $1- 1 / \poly (nk)$, we have
\begin{align*}
 2 \| a \|_2^2 \frac{kn}{B} \geq \| b \|_2^2 \geq \frac{1}{2} \| a \|_2^2 \frac{kn}{B}.
\end{align*}
\end{claim}
\begin{proof}
We consider the square of the entry-wise $\ell_2$ norm of $b$ (we define $b$ in a matrix way, but we can view it as a vector),
\begin{align*}
\| b \|_2^2 = & ~ \sum_{i=1}^k \sum_{j=1}^n b_{i,j}^2 \\
= & ~ \sum_{i=1}^k \sum_{j=1}^n \left( \sum_{p\in [B]} a_p {\bf 1}_{ h(i,j) = p } \right)^2 \\
= & ~ \sum_{i=1}^k \sum_{j=1}^n \left( \text{diagonal-term} + \text{off-diagonal-term} \right) \\
= & ~ \sum_{i=1}^k \sum_{j=1}^n \left( \sum_{p\in [B]} a_p^2 {\bf 1}_{h(i,j) = p} + \sum_{p \neq q} a_p a_q {\bf 1}_{h(i,j)=p} {\bf 1}_{h(i,j)=q} \right) \\
= & ~ \sum_{i=1}^k \sum_{j=1}^n \sum_{p\in [B]} a_p^2 {\bf 1}_{h(i,j) = p} \\
= & ~ \sum_{p\in [B]} a_p^2  \sum_{i=1}^k \sum_{j=1}^n {\bf 1}_{h(i,j) = p} \\
\geq & ~ \sum_{p\in [B]} a_p^2 \cdot \frac{1}{2} \frac{kn}{B} \\
= & ~ \frac{1}{2} \| a \|_2^2 \frac{kn}{B}, 
\end{align*}
where the first step follows by viewing $b$ as a vector, the last step follows by Lemma~\ref{lem:concentration_of_hashing_buckets}.
\end{proof}

\begin{lemma}[Concentration of hashing buckets]\label{lem:concentration_of_hashing_buckets}
Given integers $N$ and $B \lesssim N / \log N$. Let $h : [N] \rightarrow [B]$ denote a $t$-wise independent hash function such that 
\begin{align*}
\Pr_{h \sim {\cal H} }[ h(i) = j] = 1/B, \forall i \in [N], \forall j \in [B].
\end{align*}
Then,\\
$\mathrm{(\RN{1})}$, if $t = \Theta(\log N)$, with probability at least $1 - 1 / \poly(N)$, we have for all $j \in [B]$,
\begin{align*}
0.9 N/B \leq \sum_{i=1}^N {\bf 1}_{h(i) = j} \leq 1.1 N/B;
\end{align*}
$\mathrm{(\RN{2})}$, if $t = \Theta(\log B)$, with probability at least $1 - 1 / \poly(B)$, we have for all $j \in [B]$,
\begin{align*}
0.9 N/B \leq \sum_{i=1}^N {\bf 1}_{h(i) = j} \leq 1.1 N/B.
\end{align*}
\end{lemma}

\begin{proof}
First, notice that $\Pr[ h(i) = p ] = \frac{1}{B}$. We can estimate the expectation,
\begin{align*}
\E \left[ \sum_{i=1}^N h(i) = p \right] = \frac{N}{B} \gg 1.
\end{align*}

By Theorem~\ref{thm:br94_theorem} for $t$-wise, we have
\begin{align*}
\Pr \left[ \left| \sum_{i=1}^N {\bf 1}_{h(i) = p} - \frac{N}{B}  \right| > a \right] < 8 \left( \frac{t N/B + t^2}{a^2} \right)^{t/2}. 
\end{align*}
Choosing $ a = 0.1 N/B $, we have
\begin{align*}
\Pr \left[ \left| \sum_{i=1}^N {\bf 1}_{h(i) = p} - \frac{N}{B}  \right| > 0.1 N/B \right] < & ~ 8 ( 100 N/ (N/B)^2  )^{t/2} \\
= & ~ 8 \left( \frac{t N/B + t^2}{0.01 N^2/B^2 } \right)^{t/2} \\
\leq & ~ 8 \left( \frac{2t N/B}{0.01 N^2/B^2} \right)^{t/2} & \text{~by~} N/B > t \\
= & ~ 8 \left( \frac{ 200 t }{  N/B  } \right)^{t/2} \\
\leq & ~ 8 (1/10)^{t/2} & \text{~by~} N/B > 2000 t \\
\leq & ~ 1/\poly(N) & \text{~by~} t = \Theta(\log N). 
\end{align*}
The total number of $p$'s is at most $B$ and $N \geq B$, thus by taking a union bound, we prove the part (\RN{1}).

For part (\RN{2}), the only change is setting $t = \Theta(\log B)$. We can still take a union over all $p$, thus we complete the proof of part (\RN{2}).

\end{proof}

\subsection{Weight matrix is almost full rank}
In this section, we explain that making the assumption that weight matrix is full rank is reasonable in both theory and practice.

Consider the situation where $B=2$, the following result indicates that, the weight matrix will be full rank with high probability.
\begin{theorem}[\cite{bvw10}]
Let $M \in \R^{n\times n}$ denote a matrix where each entry is $+1$ with probability $1/2$ and $-1$ with probability $1/2$, then
\begin{align*}
\Pr \left[ \rank(M) = n \right] \geq 1 - \left( \frac{1}{\sqrt{2}} + o(1) \right)^n.
\end{align*}
\end{theorem}

Using Theorem 9.2 in \cite{bvw10} and Lemma in \cite{sggc14}, we have that with high probability the above Theorem also holds for general $B > 2$.

In addition, we also run experiments, we do observe that condition number of the weight matrix is not too big.


\section{Experiments}



\begin{figure*}[!h]
\centering
{\includegraphics[width=0.25\textwidth]{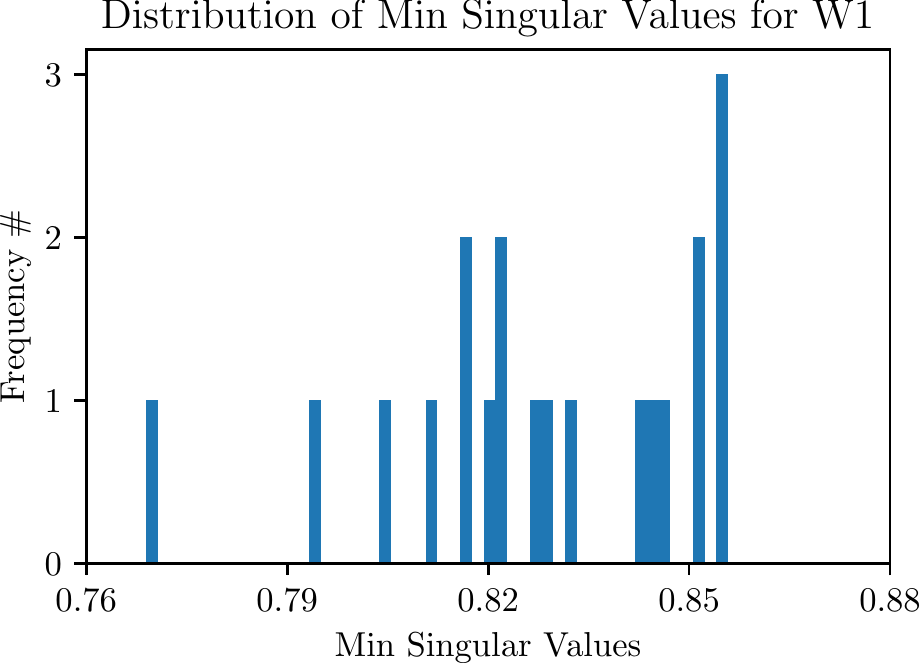}}\label{fig:Min_Singular_Values_for_W1_nhu500_CR64}%
{\includegraphics[width=0.25\textwidth]{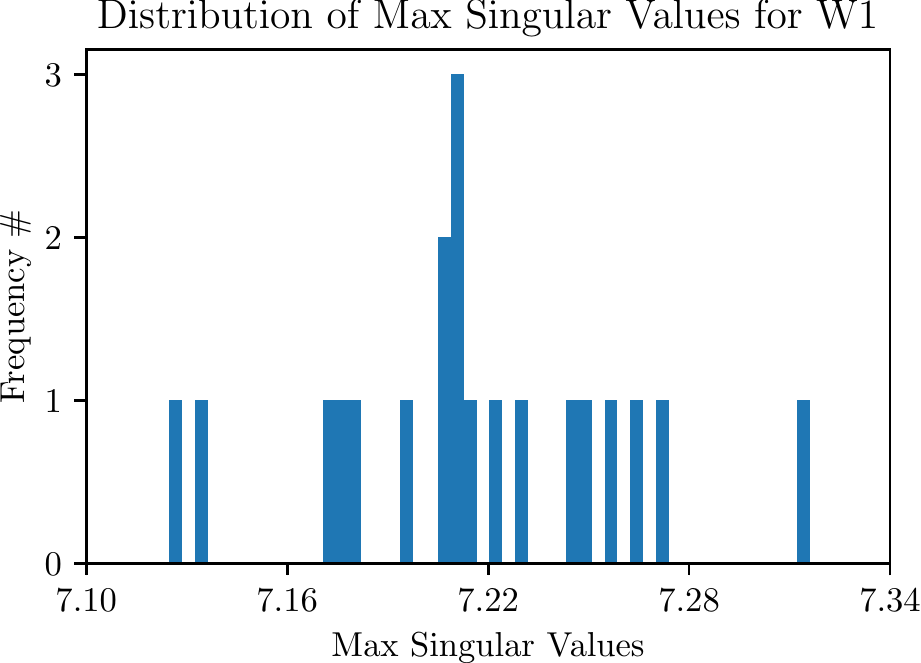}}\label{fig:Max_Singular_Values_for_W1_nhu500_CR64}%
{\includegraphics[width=0.25\textwidth]{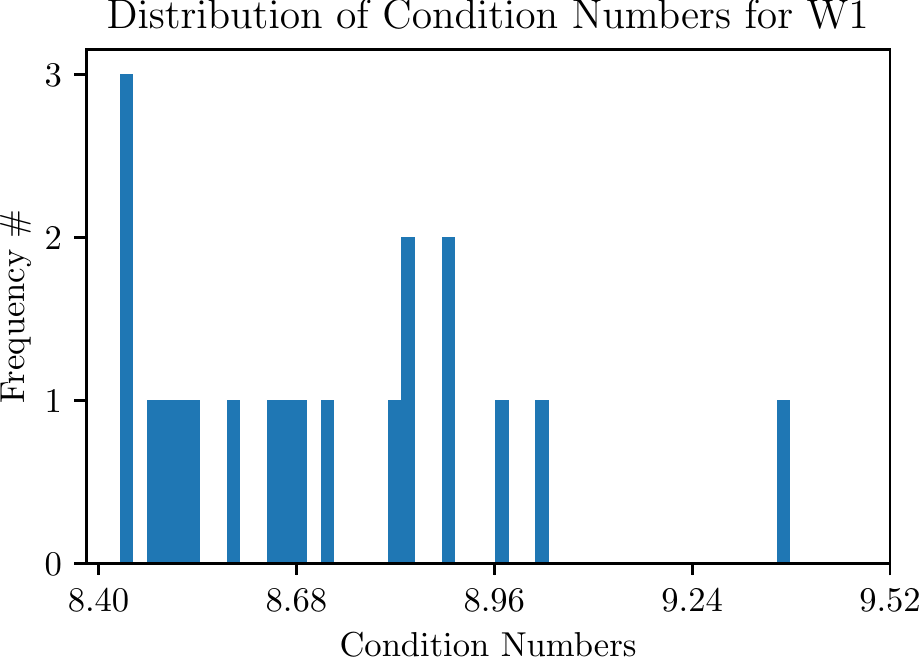}}\label{fig:Condition_Numbers_for_W1_nhu500_CR64}%
{\includegraphics[width=0.25\textwidth]{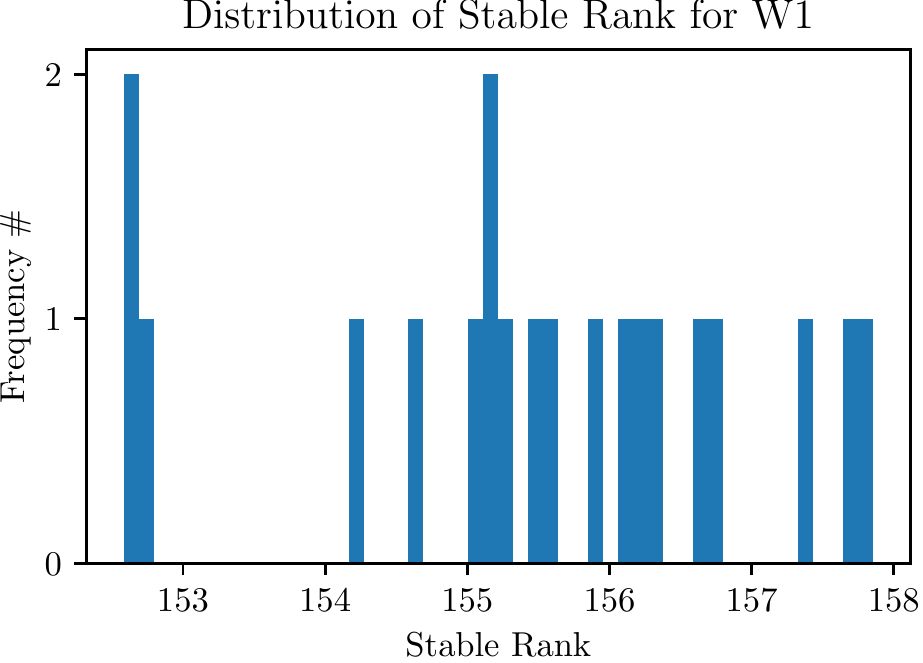}}\label{fig:Stable_Rank_for_W1_nhu500_CR64}\\
\vspace{.1in}
{\includegraphics[width=0.25\textwidth]{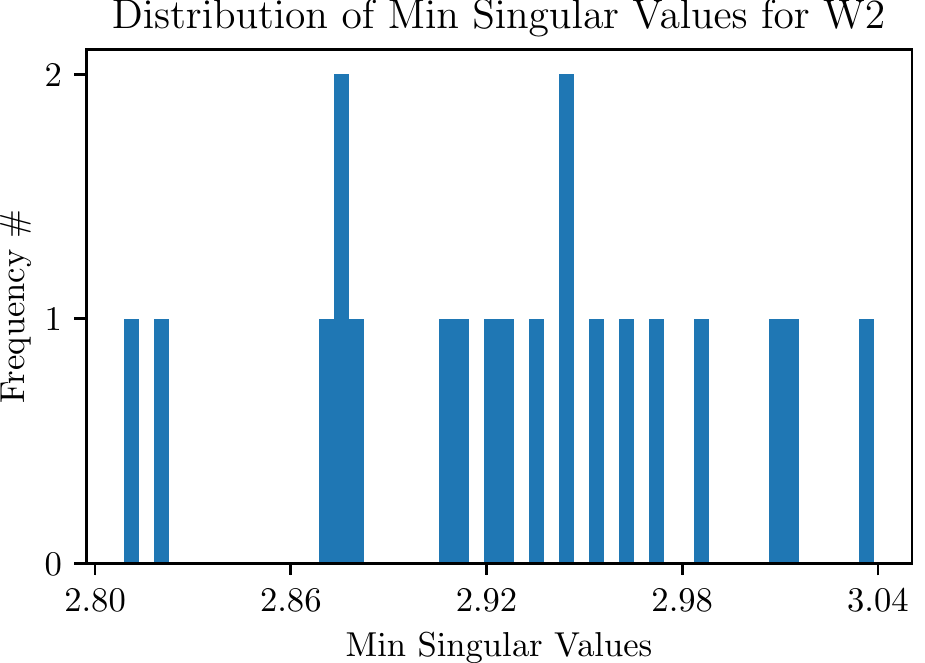}}\label{fig:Min_Singular_Values_for_W2_nhu500_CR64}%
{\includegraphics[width=0.25\textwidth]{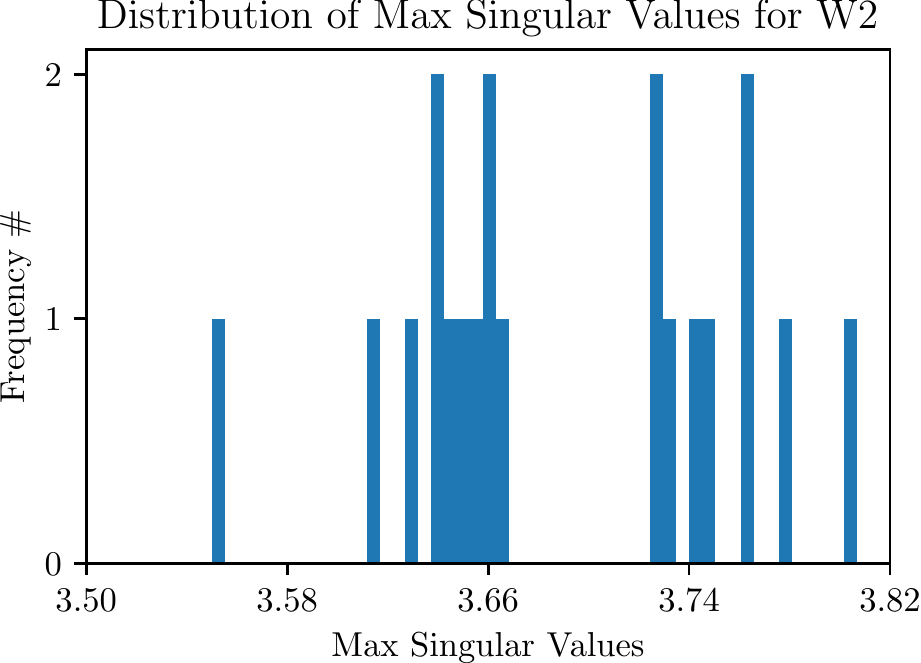}}\label{fig:Max_Singular_Values_for_W2_nhu500_CR64}%
{\includegraphics[width=0.25\textwidth]{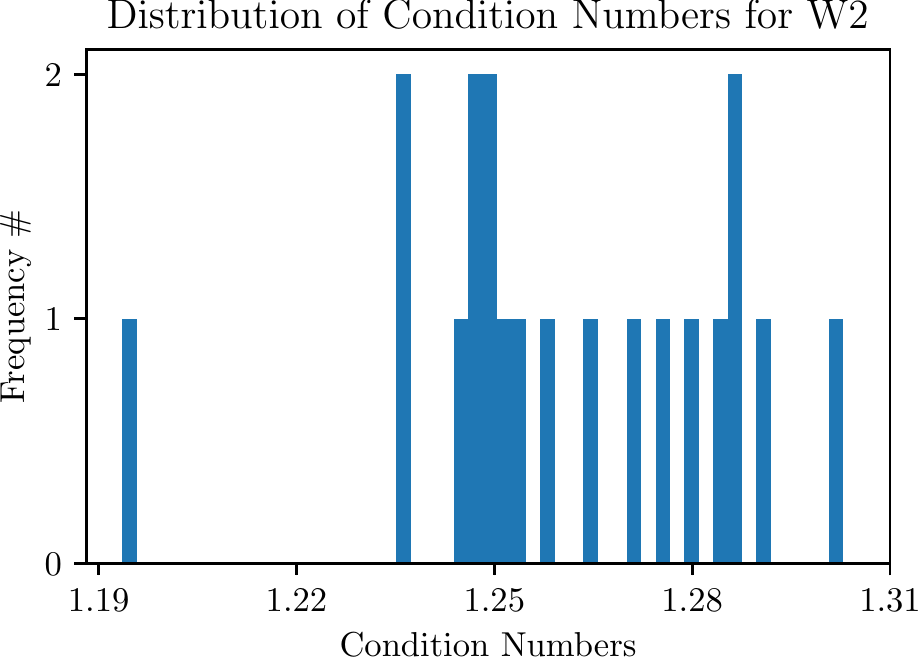}}\label{fig:Condition_Numbers_for_W2_nhu500_CR64}%
{\includegraphics[width=0.25\textwidth]{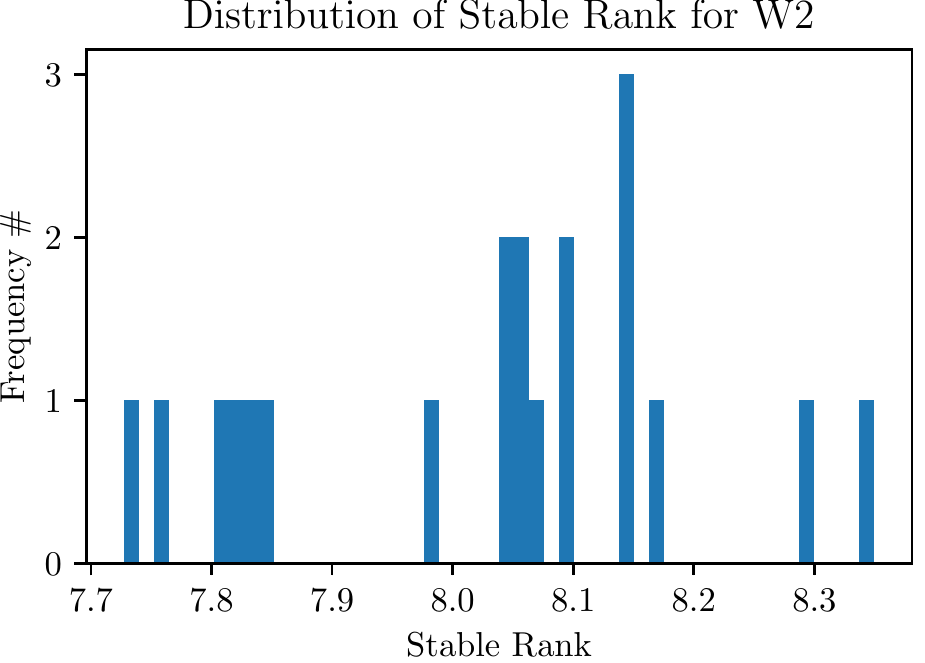}}\label{fig:Stable_Rank_for_W2_nhu500_CR64}
\caption{\small Input dimension is 784. Distributions of singular values, condition numbers, and stable ranks for two weight matrices $W_1$ and $W_2$ in HashedNets with 500 hidden units for 20 random seeds.}
\label{fig:ConditionNumbers_nhu500}
\end{figure*}




\end{document}